\documentclass{article}

\usepackage[margin=1.2in]{geometry}
\usepackage[utf8]{inputenc} 
\usepackage[T1]{fontenc}    
\usepackage{enumitem}
\usepackage{hyperref}       
\usepackage{url}            
\usepackage{booktabs}       
\usepackage{subcaption}
\usepackage{amsfonts,amsmath,amssymb,bm,mathtools,amsthm}       

\DeclareMathOperator*{\argmin}{arg\,min}

\DeclareMathOperator{\Tr}{Tr}
\usepackage{nicefrac}       
\allowdisplaybreaks
\usepackage{microtype}      
\usepackage[pdftex]{graphicx}
\newcommand\sublength{.19}
\newcommand\sublengthtwo{.22}
\title{Error Bounds for Generalized Group Sparsity}

\author{%
  Xinyu~Zhang\\
  Department of Industrial Engineering and Operations Research\\
  Columbia University\\
  New York, NY 10027 \\
  \texttt{zhang.xinyu@columbia.edu} \\
}
\date{} 
\newtheorem{theorem}{Theorem}

\newtheorem{assumption}{Assumption}
\newtheorem{remark}{Remark}
\newtheorem{lemma}{Lemma}
\newtheorem*{lemma*}{Lemma}
\newtheorem*{theorem*}{Theorem}
\newcommand{\ds}{\textup{ds}} 
\begin{document}
\setlength{\abovedisplayskip}{0pt}
\setlength{\belowdisplayskip}{0pt}
\setlength{\abovedisplayshortskip}{0pt}
\setlength{\belowdisplayshortskip}{0pt}

\maketitle
\begin{abstract}
In high-dimensional statistical inference, sparsity regularizations have shown advantages in consistency and convergence rates for coefficient estimation. We consider a generalized version of Sparse-Group Lasso which captures both element-wise sparsity and group-wise sparsity simultaneously. We state one universal theorem which is proved to obtain results on consistency and convergence rates for different forms of double sparsity regularization. The universality of the results lies in an generalization of various convergence rates for single regularization cases such as LASSO and group LASSO and also double regularization cases such as sparse-group LASSO. Our analysis identifies a generalized norm of $\epsilon$-norm, which provides a dual formulation for our double sparsity regularization.
\end{abstract}
\section{Introduction}
Sparsity regularizations, which often involves feature-wise norm such as $\ell_1$-norm, group-wise norm such as $\ell_{2}$ or both, has attracted enormous research attentions over the past decades in a wide range of research ares, including statistics \cite{tibshiraniRegressionShrinkageSelection1996}, machine learning\cite{trevorhastieElementsStatisticalLearning2009}. In a high dimensional setting, where the number of unknown coefficients is much larger than the number of observations, sparse model demonstrates its power in computational and statistical efficiency, encouraging a ubiquitous application in fields such as genetics\cite{bickelOverviewRecentDevelopments2009}, imaging and signal processing\cite{jaspanCompressedSensingMRI2015}\cite{davidl.donohoCompressedSensing2006}.

Simultaneous structured models are considered in our paper, in which the parameter of interest has multiple structures at the same time. One example of such models are sparse-group LASSO\cite{simonSparseGroupLasso2013} \cite{wangTwoLayerFeatureReduction2014} \cite{ndiayeGAPSafeScreening2016}  \cite{idaFastSparseGroup2019} \cite{caiSparseGroupLasso2019}. The regularization can be considered as a convex combination of $\ell_1$-norm for feature-wise sparsity and $\ell_{1,2}$ norm for group-wise sparsity. Like either one of the single norm regularization, the double sparsity model brings on sparse solutions, a desirable property for model selection and variables selection. However, neither of the single sparsity is able to detect active/inactive features and groups simultaneously when the overall feature space is known to have some group structures. In practice, such a structure arises in nature in variety of fields such as  groups of gene pathways in genome-wide study\cite{silverPathwaysdrivenSparseRegression2013} and factor indicators in multinomial logistic regression\cite{buhlmannStatisticsHighDimensionalData2011} and multi-task learning\cite{argyriouConvexMultitaskFeature2008}. 

Recent years have witness several research work on the statistical properties of sparse group LASSO. For instance, Chatterjee et al. \cite{chatterjeeSparseGroupLasso2012} develop a consistency result for tree-structured norm regularizers and discuss application in climates prediction, where sparse group LASSO as special case. Poignard \cite{poignardAsymptoticTheoryAdaptive2020} discussed asymptotic behavior and weak convergence results for adaptive sparse group LASSO where penalties are weighted by some random coefficients. 
Most recently, Cai et al. \cite{caiSparseGroupLasso2019} discuss the optimal theoretical guarantees for both the sample complexity and estimation error of sparse group LASSO. However, different from previous results, this paper focuses on a generalized format of group sparsity in the double sparsity term. Such a generalization covers many of the norms in the literature, including LASSO, group LASSO and spare-group LASSO as a special cases. 

Two main contributions are summarized as follow. Firstly, we introduce a new norm $\epsilon q$-norm in finite vector space which is the main tool for generalized double sparsity of interest. Leveraging duality and decomposition results on $\epsilon q$-norm, one can reformulate a dual problem and the related dual norm in a concise way. Secondly, choices of parameters (penalty level) that recovers simultaneous sparsity structures are derived and the error bounds for estimators with the generalized group sparsity regularization is investigated, the special cases of which, including LASSO, group LASSO and sparse Group LASSO, match with past research results. 



\section{Preliminary}
\paragraph{Notation.}
$S_{\alpha}(\cdot)$ is the soft-thresholding function such that $S_\alpha(x)=\text{sgn}(x)(|x|-\alpha)_+$ for any $x\in \mathbb{R}$. 
For any set $S$, $S^c$ denotes its complement and $|S|$ denotes its cardinality. 
Throughout the paper, we focus on the parameter index set $\{1,...,p\}$ partitioned into $G$ groups. Denote $(1),...,(G)\subseteq \{1,...,p\}$ as the index sets belonging to each group. 
Denote the support set $\textbf{Supp}(\beta)=\{i|\beta_i\neq 0\}$ and group support set $\textbf{GSupp}(\beta)=\{g|\|\beta_{(g)}\|_2\neq 0\}$. 
$\|\beta\|_q:=\big(\sum\limits_{i}|\beta_i|^q\big)^{\nicefrac{1}{q}}$ denotes the $\ell_q$ norm of vector $\beta$. 
$\lambda_{\min}(X)$ and $\lambda_{\max}(X)$ denotes the smallest and largest eigenvalues of matrix $X$ respectively.  
\paragraph{Generalized double sparse model.} In this paper, we consider in linear regression model in high-dimensional data setting in which $p$-dimensional covariates $X$ and coefficients $\beta$ follow the group structure:  
\begin{align}\label{gstruct}
X = [X_{(1)},...,X_{(G)}], \beta = \Big((\beta_{(1)})^\top,...,(\beta_{(G)})^\top\Big)^\top, X_{(g)}\in \mathbb{R}^{n\times p_g},\beta_{(g)}\in \mathbb{R}^{(p_g)}.
\end{align}
where $n$ is the sample size and $p$-dimensional space is divided into $G$ known groups, where the $g$th group contains $p_g$ variables. 

Consider the following double sparsity problem to estimate coefficients $\beta$
\begin{equation}\label{DS-LASSO}
\begin{aligned}
& \underset{\beta}{\text{minimize}}
& & \|y-X\beta\|_2^2+\lambda (\tau\|\beta\|_1+(1-\tau)\sum\limits_{g=1}^G w_g\|\beta_{(g)}\|_{\alpha_g})\\
\end{aligned}
\end{equation}
where $y\in \mathbb{R}^n$ is the dependent vector, $X\in \mathbb{R}^{n\times p}$ is the input design matrix of explanatory variables, $\beta\in \mathbb{R}^p$ is the vector of regression coefficients, and $\lambda,\tau,w_g$ are positive tuning parameters that control the regularization on element-wise and group-wise sparsity. In particular, $\lambda$ controls the overall simultaneous regularization, balancing between square loss of data fitting and regularization term. $\tau$ dictates the trade-off between group-wise and element-wise sparsity ($\tau\in [0,1]$). For each group $g$, the regularization term on group structure is given as $\|\beta_{(g)}\|_{\alpha_g}$ which is a $\ell_{\alpha_g}$-norm with $\alpha_g\geq 1$. We denote a vector of norm parameters $\alpha = (\alpha_1,...,\alpha_G)\in [1,\infty)^{G}$. We let $\alpha\equiv c$ denote the case $\alpha_g=c,\forall g\geq 1$. One will reduce to the LASSO \cite{tibshiraniRegressionShrinkageSelection1996} when $\tau =1$ and the Group-LASSO \cite{yuanModelSelectionEstimation2006} with $\tau=0, \alpha\equiv2$. For $\tau\neq 0, \alpha\equiv2$, the problem reduces to Sparse Group LASSO (see for example \cite{simonSparseGroupLasso2013} \cite{wangTwoLayerFeatureReduction2014} \cite{ndiayeGAPSafeScreening2016}  \cite{idaFastSparseGroup2019} \cite{caiSparseGroupLasso2019}) which has attracted much interests in recent years. 
\section{Norm of double sparsity and its dual}
Note that the double sparsity term given in problem \ref{DS-LASSO} is a norm as it is a positive linear combination of norms. For simplicity,  we denote this norm as $\|\cdot\|_{\ds}$, i.e. 
\begin{align}
\|\beta\|_{\ds} = \tau\|\beta\|_1+(1-\tau)\sum\limits_{g=1}^G w_g\|\beta_{(g)}\|_{\alpha_g}
\end{align}
Recall the definition of decomposability (see Definition 1 of \cite{negahbanUnifiedFrameworkHighDimensional2012a})  of a norm-based penalty term, the regularizer $\|\cdot\|_\ds$ is decomposable for some pairs of subspaces $\mathcal{S}\subseteq \mathcal{T}$, i.e. 
\begin{align}
\|\beta_{S}+\beta_{T}\|_\ds = \|\beta_{S}\|_\ds+\|\beta_{T}\|_\ds, \forall S \in \mathcal{S},T\in \mathcal{T}^c.
\end{align}
Some examples of $(\mathcal{S},\mathcal{T})$ are $\big(\textbf{Supp}(\beta),\textbf{Supp}(\beta)\big)$, $\Big(\textbf{Supp}(\beta),\big(\textbf{GSupp}(\beta)\big)\Big)$.\\ 
In order to analyze $\|\cdot\|_\ds$, we first introduce a parametric family of norms with the parameter $\epsilon\in [0,1]$ and $q\geq 1$ which we denote as $\|\cdot\|_{\epsilon q}$ and call the $\epsilon q$-norm. It is a generalized form of $\epsilon$-norm (denoted as $\|\cdot\|_{\epsilon}$), which was initially introduced by Burdakov in 1988 \cite{olegburdakovNewVectorNorm1988} in optimization literature and further developed in \cite{merkulovMethodsSolvingSystems1993} and \cite{olegburdakovNewNormData2002}). It was shown to have computation advantages in optimization procedures when applied in nonlinear data fitting, nonlinear programming and other optimization problems.

The value of $\epsilon$-norm for $x\in \mathbb{R}^p$ is given by the unique nonnegative solution of the following equivalent equations \cite{olegburdakovNewNormData2002}: 
\begin{alignat}{2}
& &\sum\limits_{i=1}^p(|x_i|-(1-\epsilon)v)_+^2&=(\epsilon v)^2\\
&\text{ or } &\|S_{(1-\epsilon) v}(x)\|_2^2&=(\epsilon v)^2.
\end{alignat}
Similarly, the value of $\epsilon q$-norm for $x\in \mathbb{R}^p$ is given by the unique nonnegative solution of the following equations (See Figure \ref{eqillustration} for an illustration.)
\begin{align}\label{defepsq}
\begin{cases}\|S_{(1-\epsilon) v}(x)\|_q^q\ \ =(\epsilon v)^q,\quad  &q<\infty,\\
\|S_{(1-\epsilon) v}(x)\|_\infty=\epsilon v, &q=\infty\end{cases}.
\end{align}

As can be seen from the definition, $q=2$ reduces to ordinary $\epsilon$-norm. Note that this norm does not belong to the set of $\ell_p$ norms. Nevertheless, $\ell_q$ and $\ell_\infty$ are two special cases of this parametric norm with $\epsilon$ chosen as $1$ and $0$ respectively. The proof that $\epsilon$-norm indeed is a vector norm was given in \cite{olegburdakovNewNormData2002} along with the formula for its dual norm ($\|x\|_\epsilon^*=\epsilon\|x\|_2+(1-\epsilon)\|x\|_1$). We show that $\epsilon q$-norm is also a norm and give some properties (bounds, norm decomposition) as well as its dual norm. 
\begin{lemma}[Vector Norm]\label{lemma_isnorm}
For any $\epsilon\in (0,1]$ and $q\geq 1$, the unique nonnegative solution $\nu\in \mathbb{R}^+$ of equation (\ref{defepsq}) defines a vector norm in $\mathbb{R}^p$.
\end{lemma}
\begin{remark}
For $\epsilon\rightarrow 0+$, it is reasonable to define the $\epsilon q$-norm as $\ell_\infty$-norm as $\|\cdot\|_{\epsilon q}\to \|\cdot\|_\infty$. We require $q\geq 1$ to be a valid vector norm as one can check that for $0< q < 1$, the resulting solution $\nu$ for each $x$ is not subadditive, therefore it does not define a norm. For $q=0$, one may reduce to a function that is analogous to $\ell_0$-"norm", capturing some forms of sparsity of vectors, i.e. for $\|x\|_{\epsilon 0}\leq 1$, either one of $x_i> 1-\epsilon$ can be hold but they can not be hold simultaneously.
\end{remark}
It is a well-known fact that any two norms in some finite-dimensional space are equivalent in the sense that they are always within a constant factor of one another. To be more specific, given two norms $\|\cdot\|_a,\|\cdot\|_b$ in $V$, $\exists 0<c_1\leq c_2$ such that $c_1\|x\|_a\leq\|x\|_b\leq c_2\|x\|_a,\forall x\in V$. In the following lemma we give the specific values of $c_1,c_2$ for $\epsilon q$-norm with respect to various $\ell_q$-norm.
\begin{lemma}[Bounds]\label{lemma_tightbound}
For any $\epsilon\in (0,1]$ and $q\geq 1$, the following tight bounds of the $\epsilon q$-norm holds for any $x\in \mathbb{R}^p$,
\begin{alignat}{2}
\frac{\|x\|_q}{p^{1/q}(1-\epsilon)+\epsilon}&\leq \|x\|_{\epsilon q}&&\leq \|x\|_q\\
\frac{\|x\|_2}{p^{(1/2-1/q)_+}(p^{1/q}(1-\epsilon)+\epsilon)} &\leq \|x\|_{\epsilon q}&&\leq p^{(1/q-1/2)_+}\|x\|_2\\
\|x\|_\infty&\leq \|x\|_{\epsilon q}&&\leq \frac{p^{1/q}\|x\|_\infty}{p^{1/q}(1-\epsilon)+\epsilon}
\end{alignat}
\end{lemma}
\begin{remark}
Similarly, one can show that for $0\leq q<1$, we have $\|\cdot\|_\infty\leq \|\cdot\|_{\epsilon q}\leq \|\cdot\|_1\leq \|\cdot\|_q$. Figure \ref{varq_ell} gives an illustrations of the norm inequalities. 
\end{remark}
\begin{figure}
  \centering
  \begin{subfigure}{0.32\textwidth}
    \centering
    \includegraphics[width=1\linewidth]{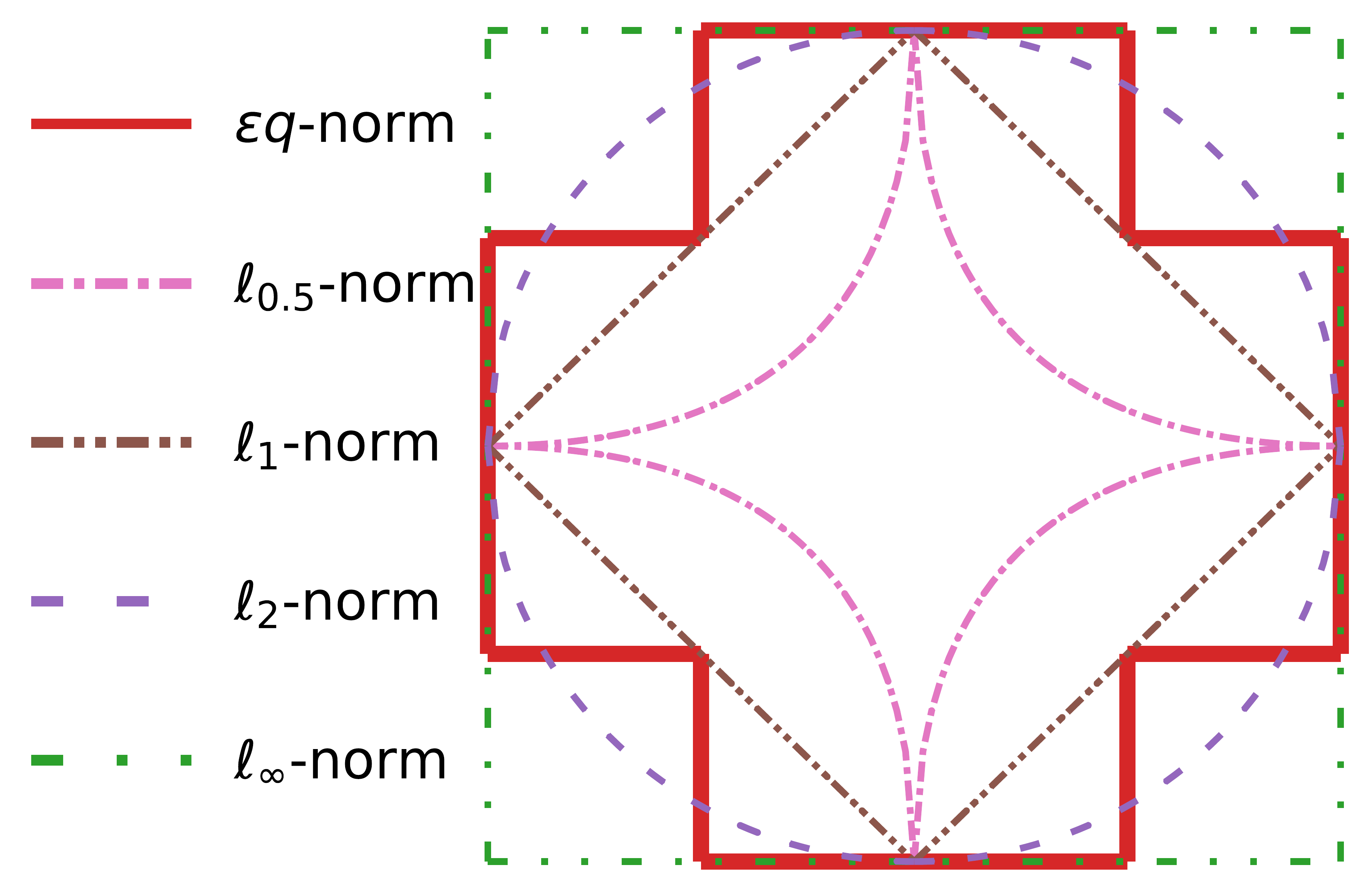}
    \caption{$q=0,\epsilon=0.5$}  
  \end{subfigure}
  \begin{subfigure}{\sublengthtwo\textwidth}
    \centering
    \includegraphics[width=1\linewidth]{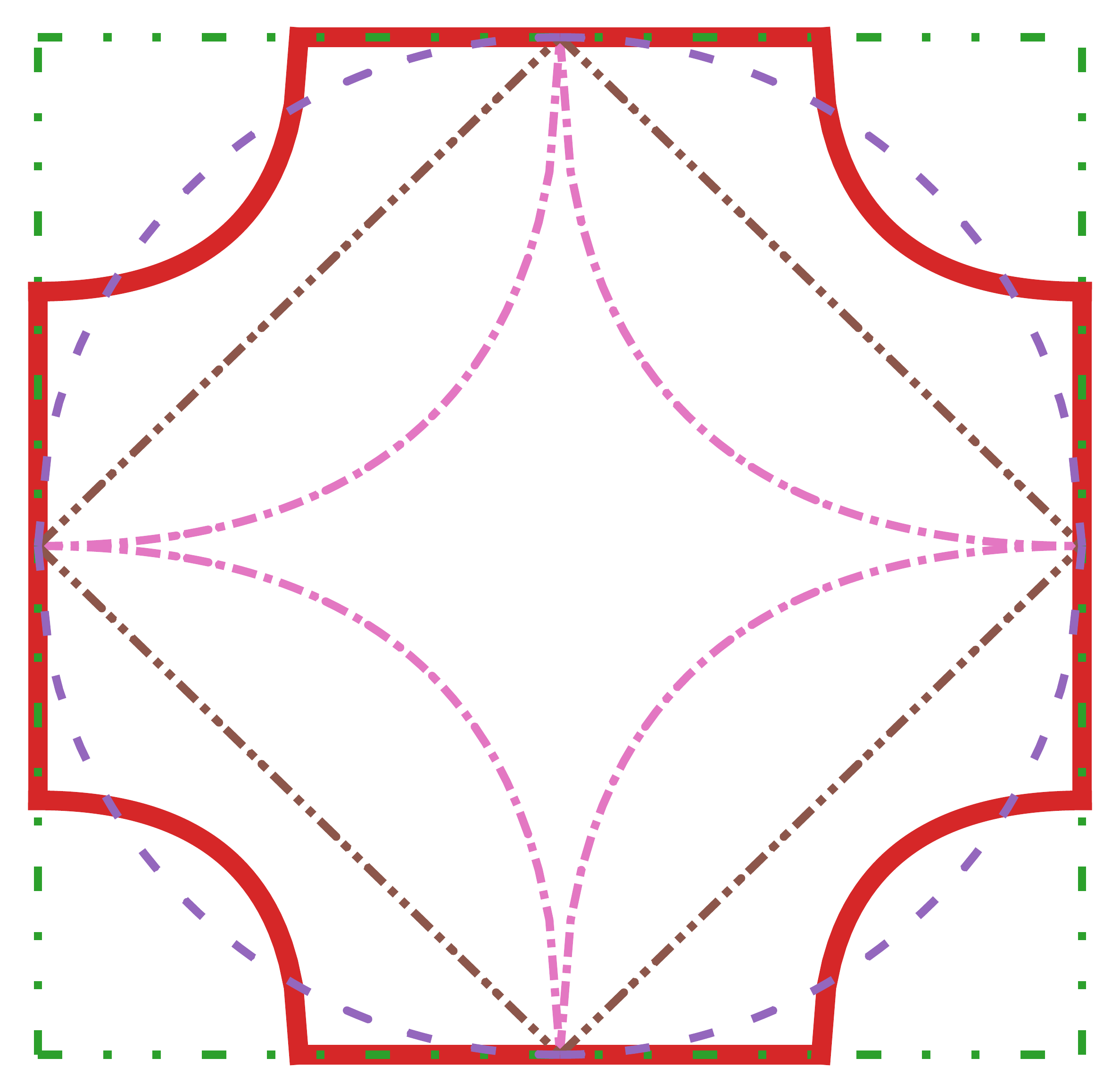}
    \caption{$q=0.5,\epsilon=0.5$}  
  \end{subfigure}
  \begin{subfigure}{\sublengthtwo\textwidth}
    \centering
    \includegraphics[width=1\linewidth]{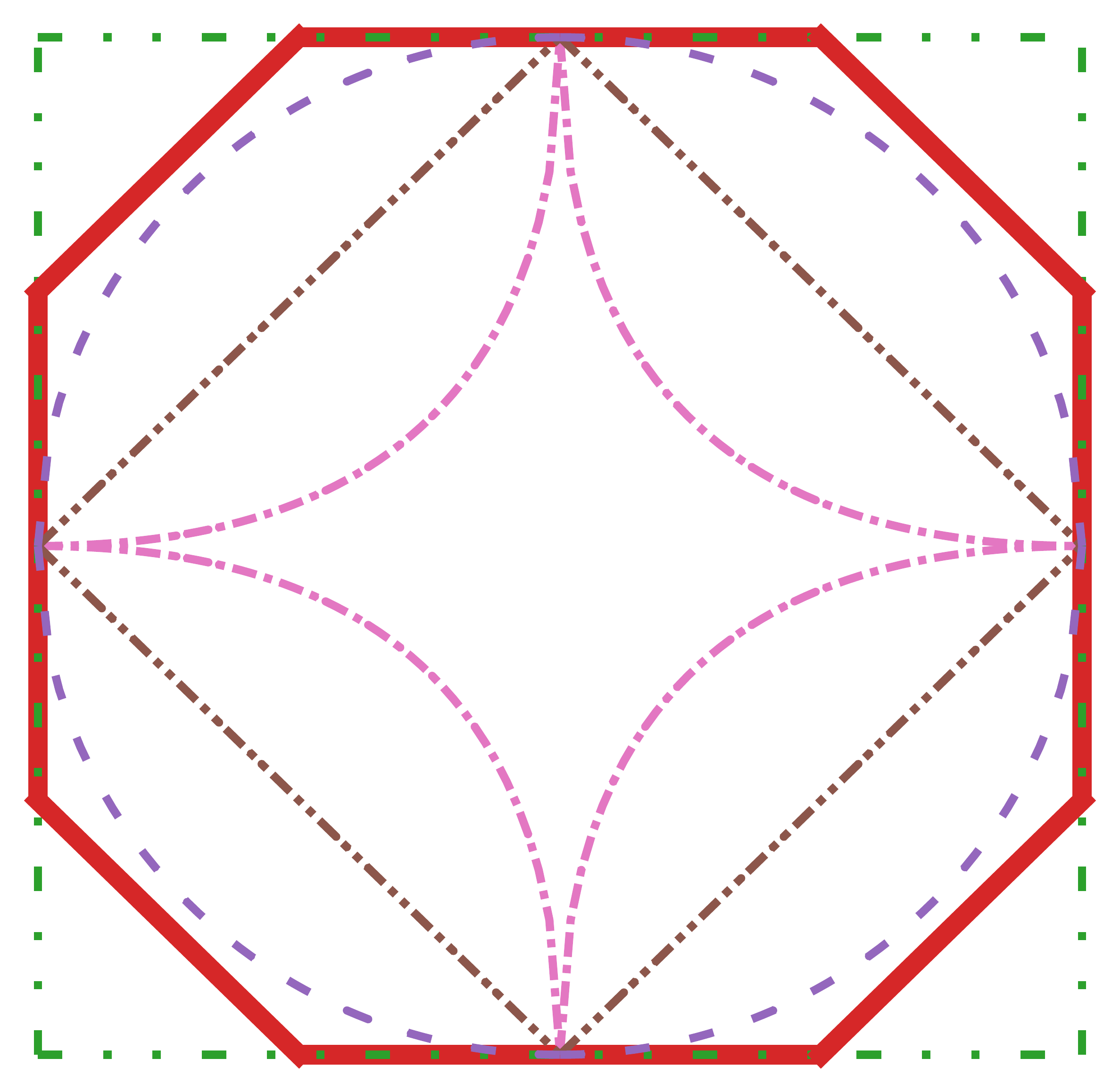}
    \caption{$q=1,\epsilon=0.5$}  
  \end{subfigure}  
  \begin{subfigure}{\sublengthtwo\textwidth}
    \centering
    \includegraphics[width=1\linewidth]{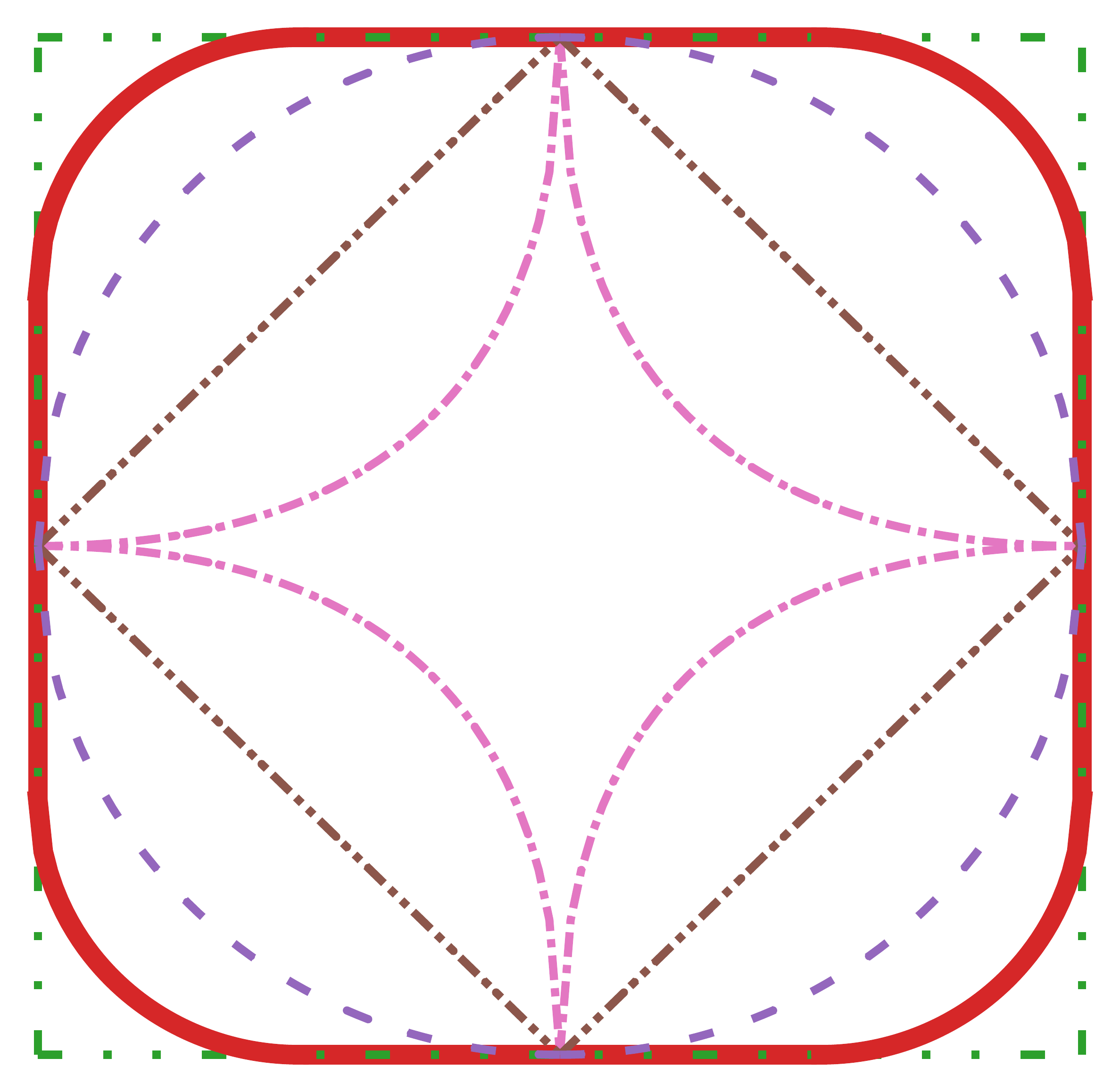}
    \caption{$q=2,\epsilon=0.5$}  
  \end{subfigure}  
  \caption{Unit ball of $\|\cdot\|_{\epsilon q}$ and $\ell_q$. The plots show the set inclusion relationship between $\|\cdot\|_{\epsilon q}$ and various $\ell_q$ unit ball.}
  \label{varq_ell}
  \end{figure}

Given $\epsilon q$-norm's relation with $\ell_q$-norm, it turns out that any vector $x$ is an addition of two vectors such that $\|x\|_{\epsilon q}$ is a convex combination (with coefficient $\epsilon$) of $\epsilon q$-norm from each of the two vectors.
\begin{lemma}[$\varepsilon$-decomposition]\label{lemma_decomposition}
Any vector $x\in \mathbb{R}^p$ can be uniquely decomposed and written in the form 
\begin{align}
x = x^{(\epsilon,q)}+x^{(1-\epsilon,q)} 
\end{align}
with $x^{(\epsilon,q)},x^{(1-\epsilon,q)}\in \mathbb{R}^p$ such that 
\begin{equation}
\begin{aligned}
\|x^{(\epsilon,q)}\|_q&=\epsilon\|x\|_{\epsilon q},\\
\|x^{(1-\epsilon,q)}\|_\infty &= (1-\epsilon)\|x\|_{\epsilon q}.
\end{aligned} 
\end{equation}
In addition, we have 
\begin{align}\label{setrep}
\{x\in\mathbb{R}^p:\|x\|_{\epsilon q}\leq \nu\}= \{u+v:u,v\in\mathbb{R}^p,\|u\|_q\leq \epsilon \nu,\|v\|_\infty\leq (1-\epsilon) \nu\}. 
\end{align}
\end{lemma}
\begin{remark}
As a matter of fact, the results also hold true for $0\leq q<1$. See the supplement for a detailed proof for all $q\geq 0$.
\end{remark}
\begin{remark}
The set representation (\ref{setrep}) gives a geometric interpretation for $\|\cdot\|_{\epsilon q}$ norm. The unit ball of $\|\cdot\|_{\epsilon q}$ (where $\nu=1$) can be written as a Minkowski addition of $\ell_q$-norm and $\ell_\infty$-norm balls of the radius $\epsilon$ and $1-\epsilon$ respectively. Figure \ref{eqillustration} shows the unit ball of $\|\cdot\|_{\epsilon q}$ for various choices of $\epsilon$ and $q$. 
\end{remark}

The $\epsilon$-decomposition gives a nice corollary for the dual norm of $\|\cdot\|_{\epsilon q}$ and the resulting norm allows us to analyze the double sparsity regularization. Basically, the dual norm can be considered as a convex combination (with coefficient $\epsilon$) of $\ell_{q^*}$ and $\ell_1$-norm where $q^*=q/(q-1)$.
\begin{figure}
  \centering
  \begin{subfigure}{\sublength\textwidth}
    \centering
    \includegraphics[width=1\linewidth]{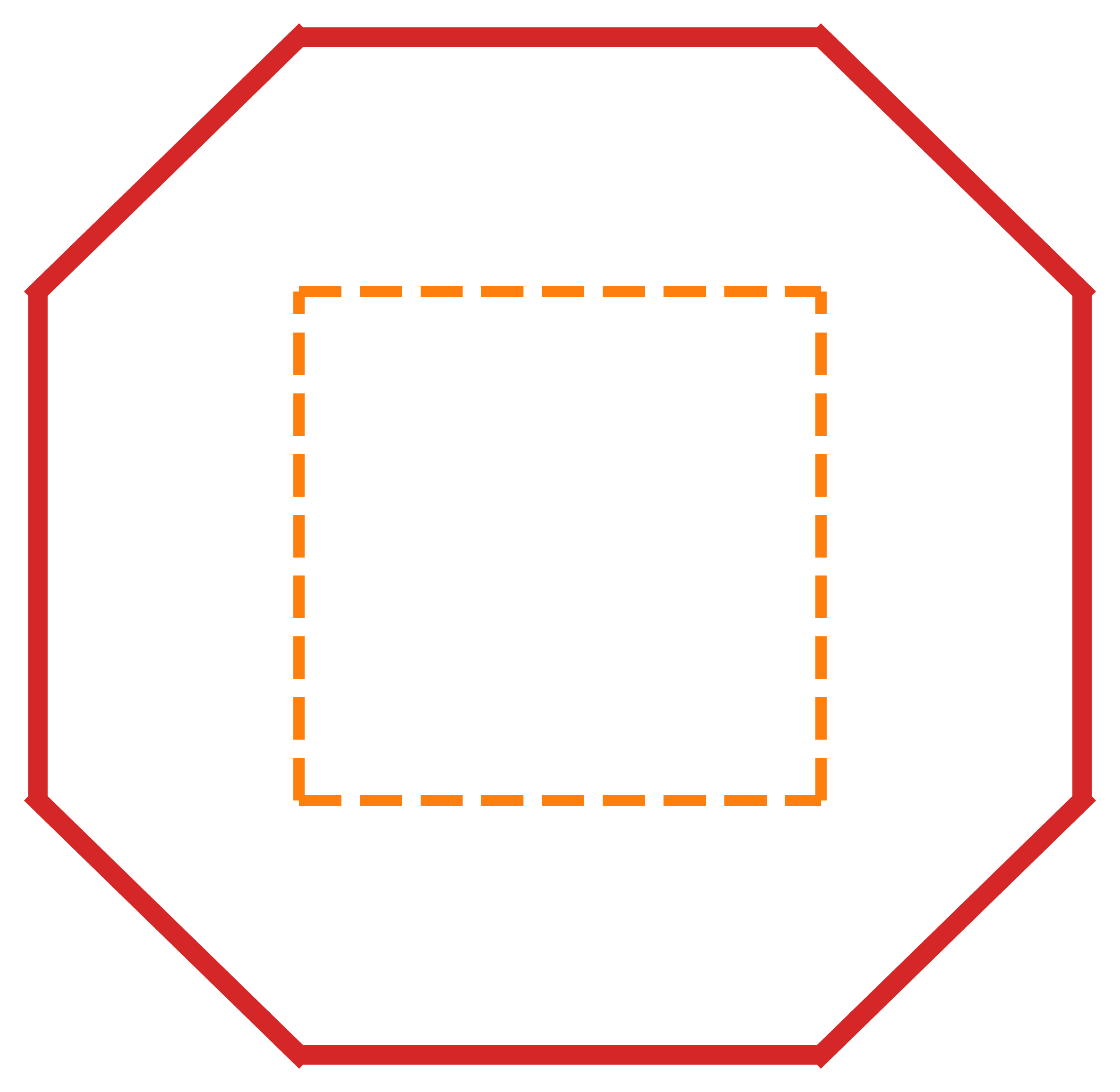}
    \caption{$q=1,\epsilon=0.5$}  
  \end{subfigure}
  \begin{subfigure}{\sublength\textwidth}
    \centering
    \includegraphics[width=1\linewidth]{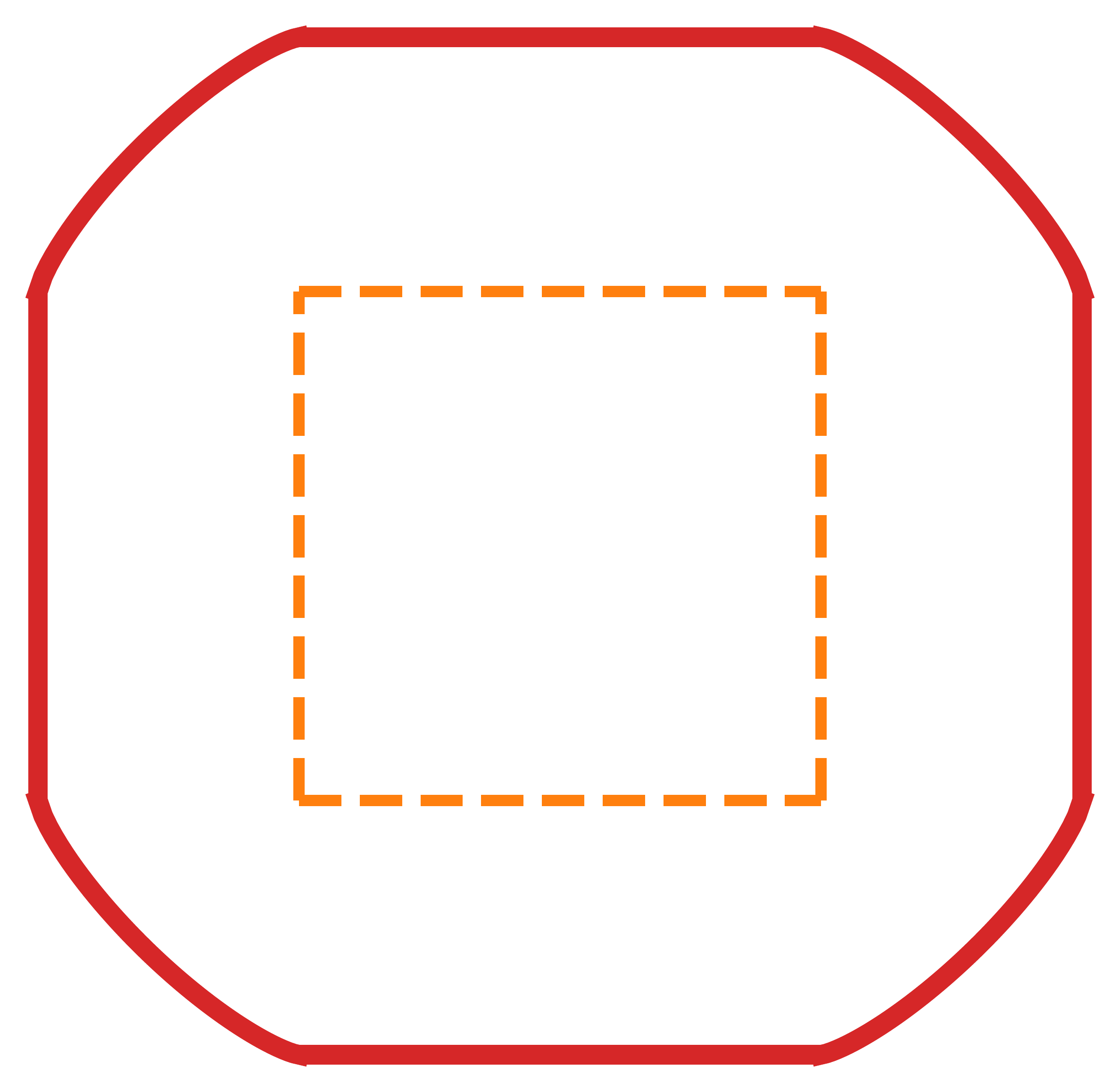}
    \caption{$q=1.3,\epsilon=0.5$}  
  \end{subfigure}
  \begin{subfigure}{\sublength\textwidth}
    \centering
    \includegraphics[width=1\linewidth]{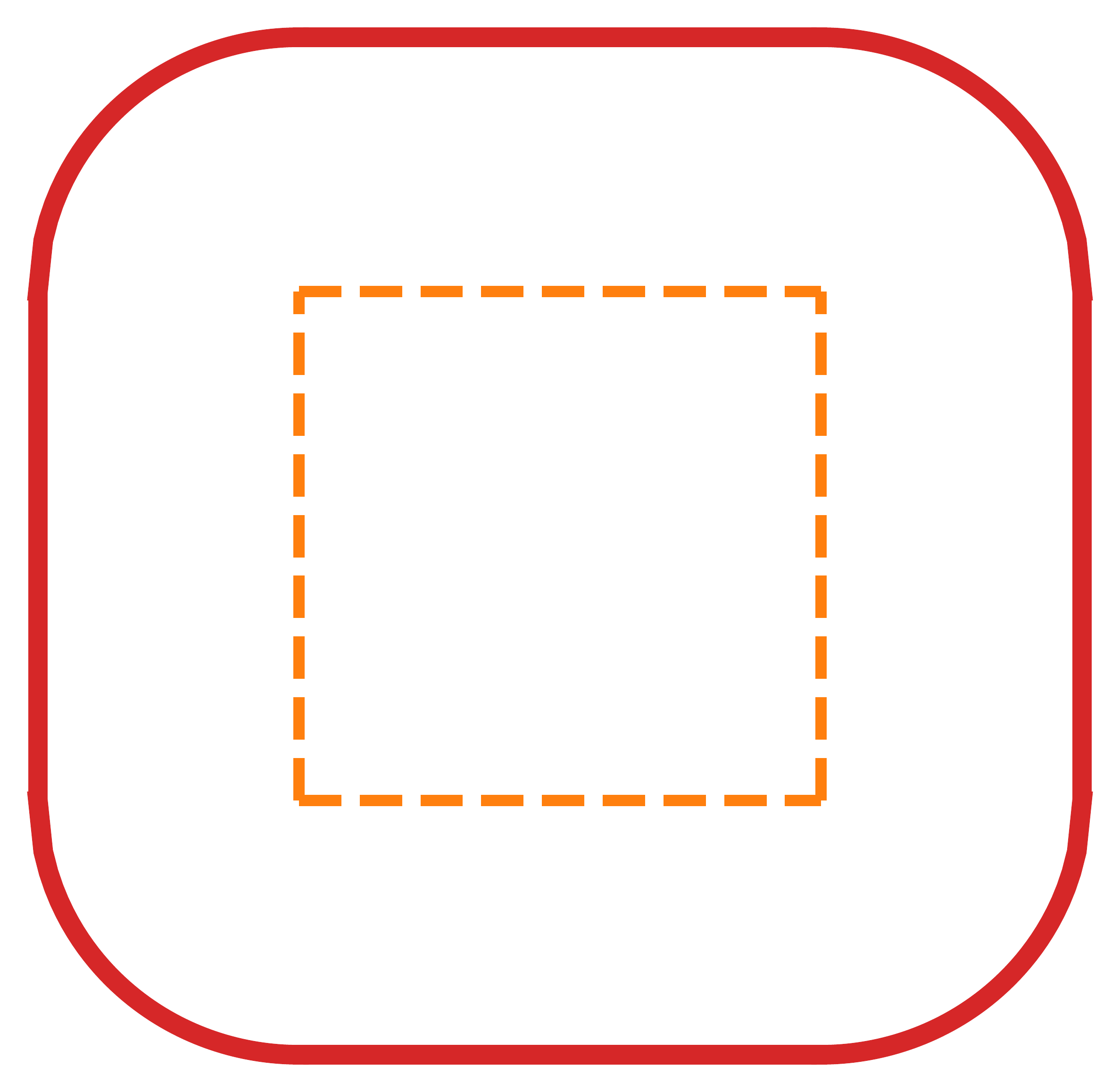}
    \caption{$q=2,\epsilon=0.5$}  
  \end{subfigure}  
  \begin{subfigure}{\sublength\textwidth}
    \centering
    \includegraphics[width=1\linewidth]{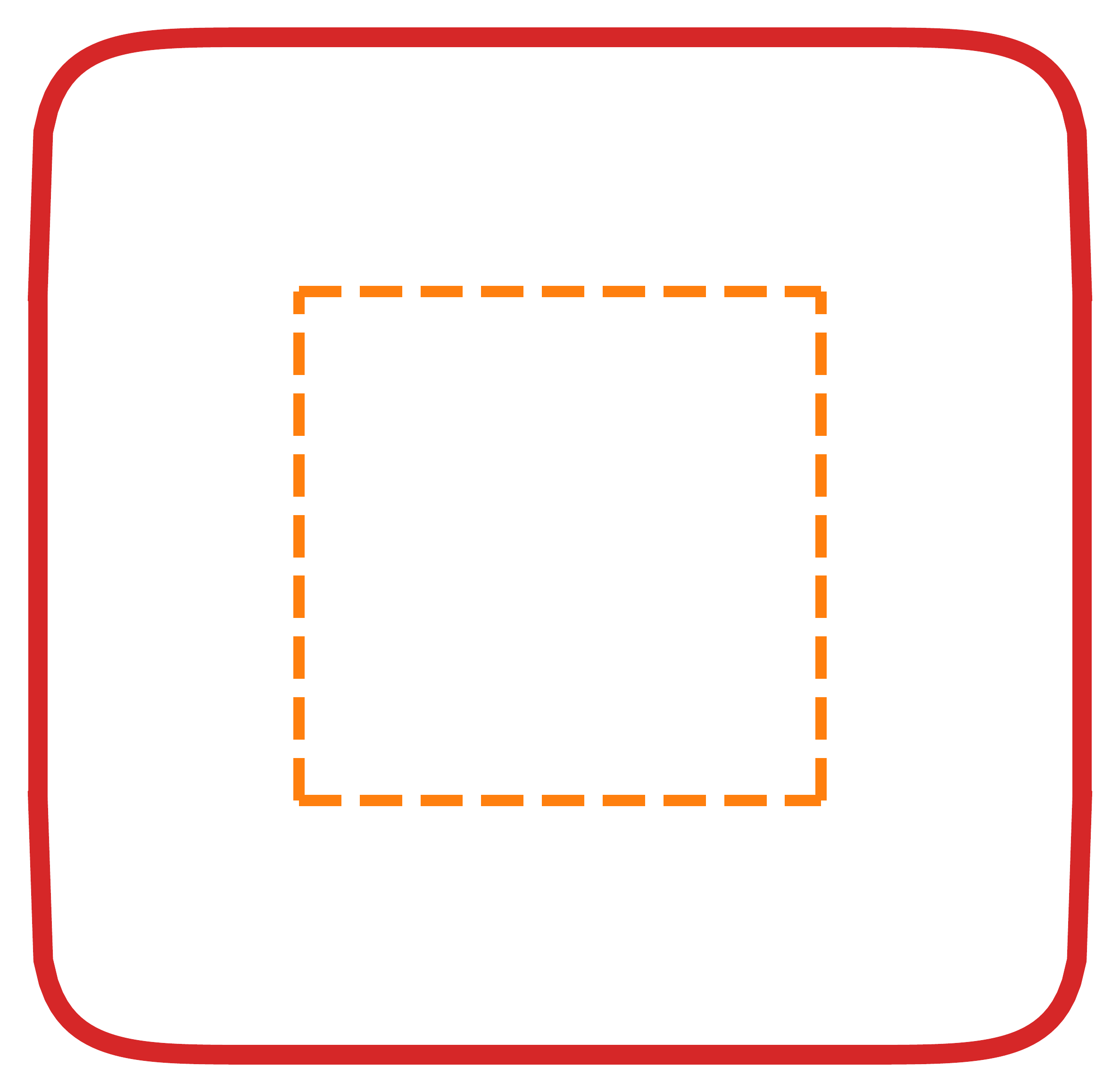}
    \caption{$q=5,\epsilon=0.5$}  
  \end{subfigure}  
  \begin{subfigure}{\sublength\textwidth}
    \centering
    \includegraphics[width=1\linewidth]{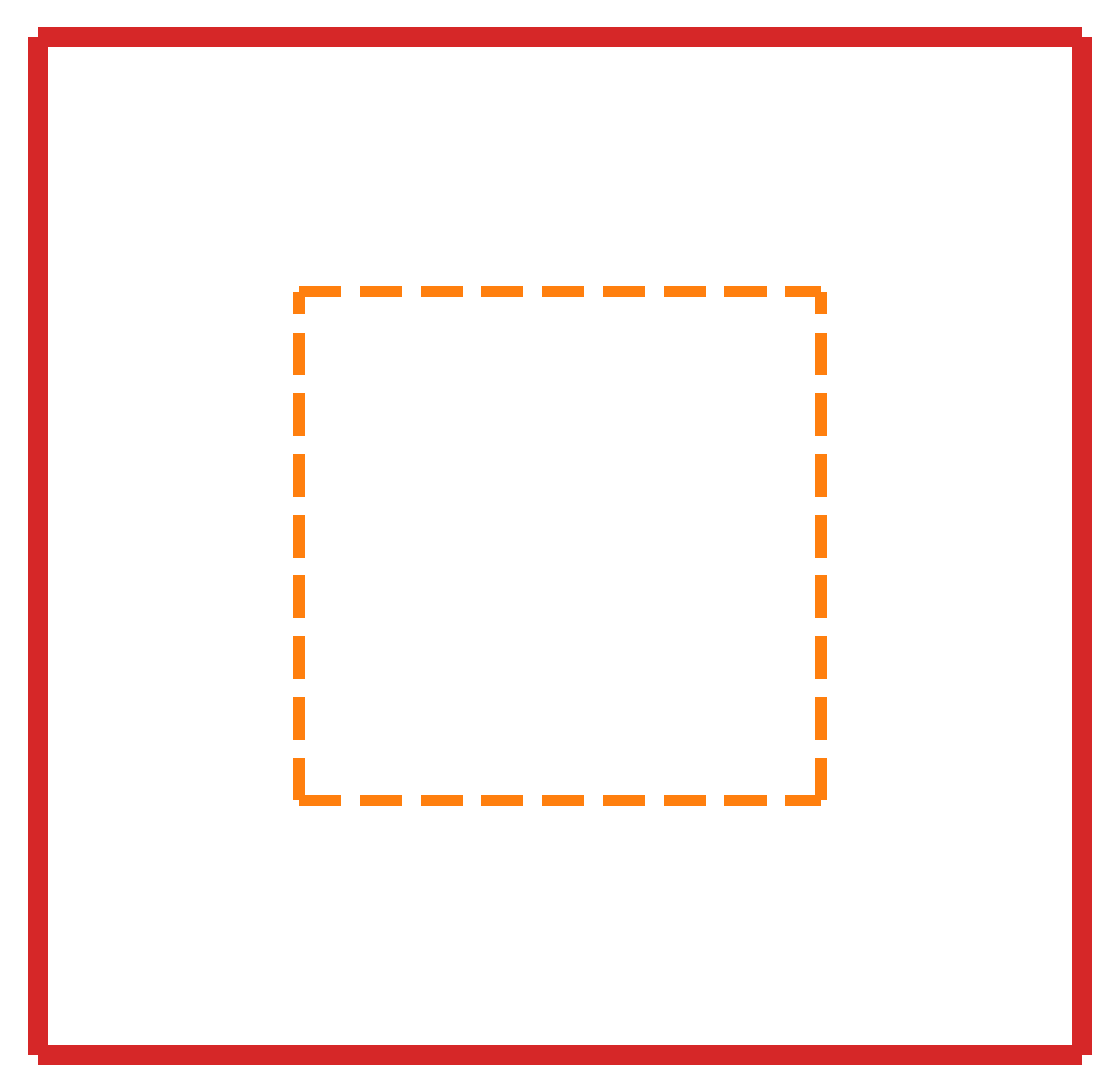}
    \caption{$q=\infty,\epsilon=0.5$}  
  \end{subfigure}  
  \begin{subfigure}{\sublength\textwidth}
    \centering
    \includegraphics[width=1\linewidth]{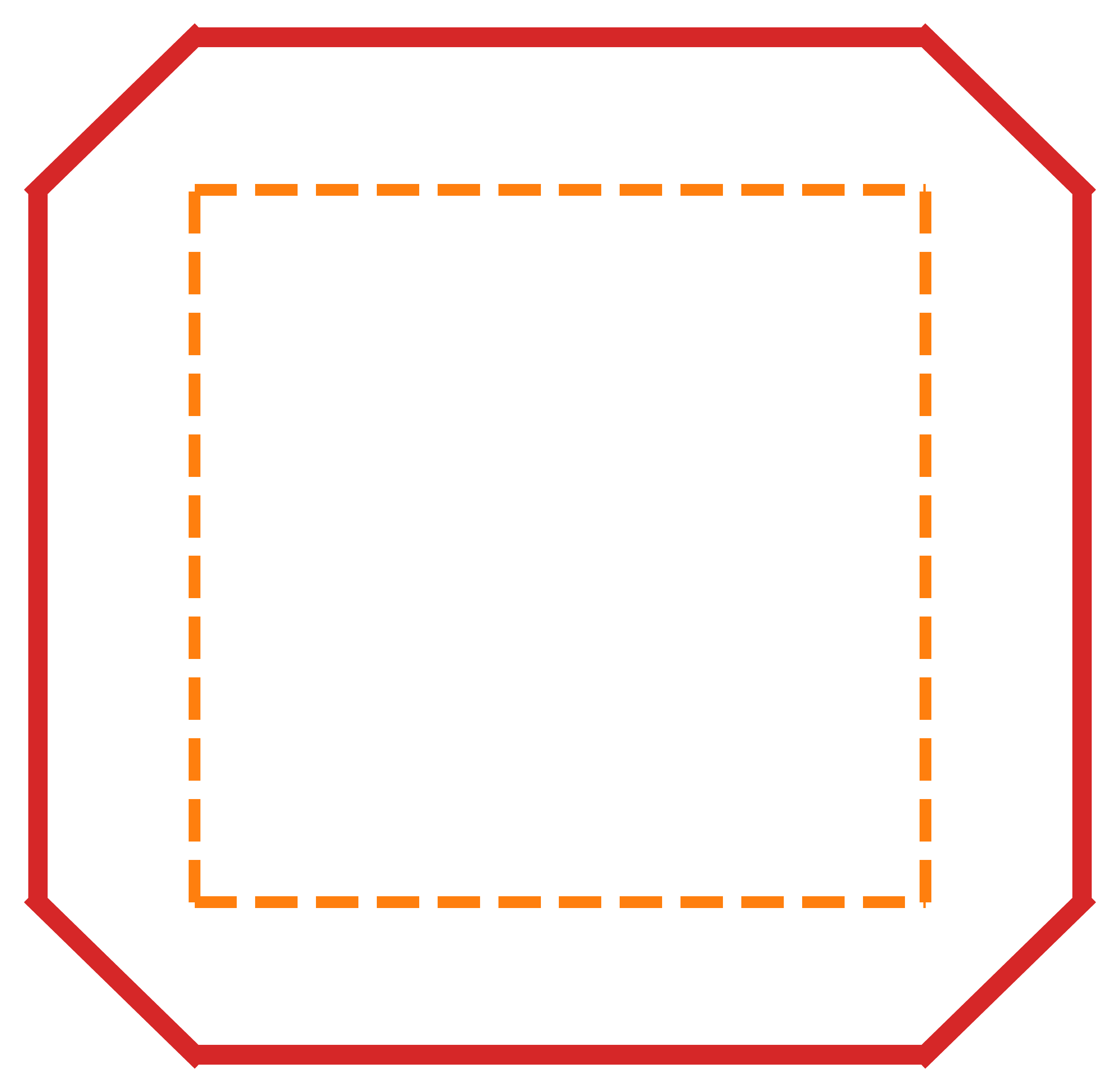}
    \caption{$\epsilon=0.3,q=1$}  
  \end{subfigure}  
  \begin{subfigure}{\sublength\textwidth}
    \centering
    \includegraphics[width=1\linewidth]{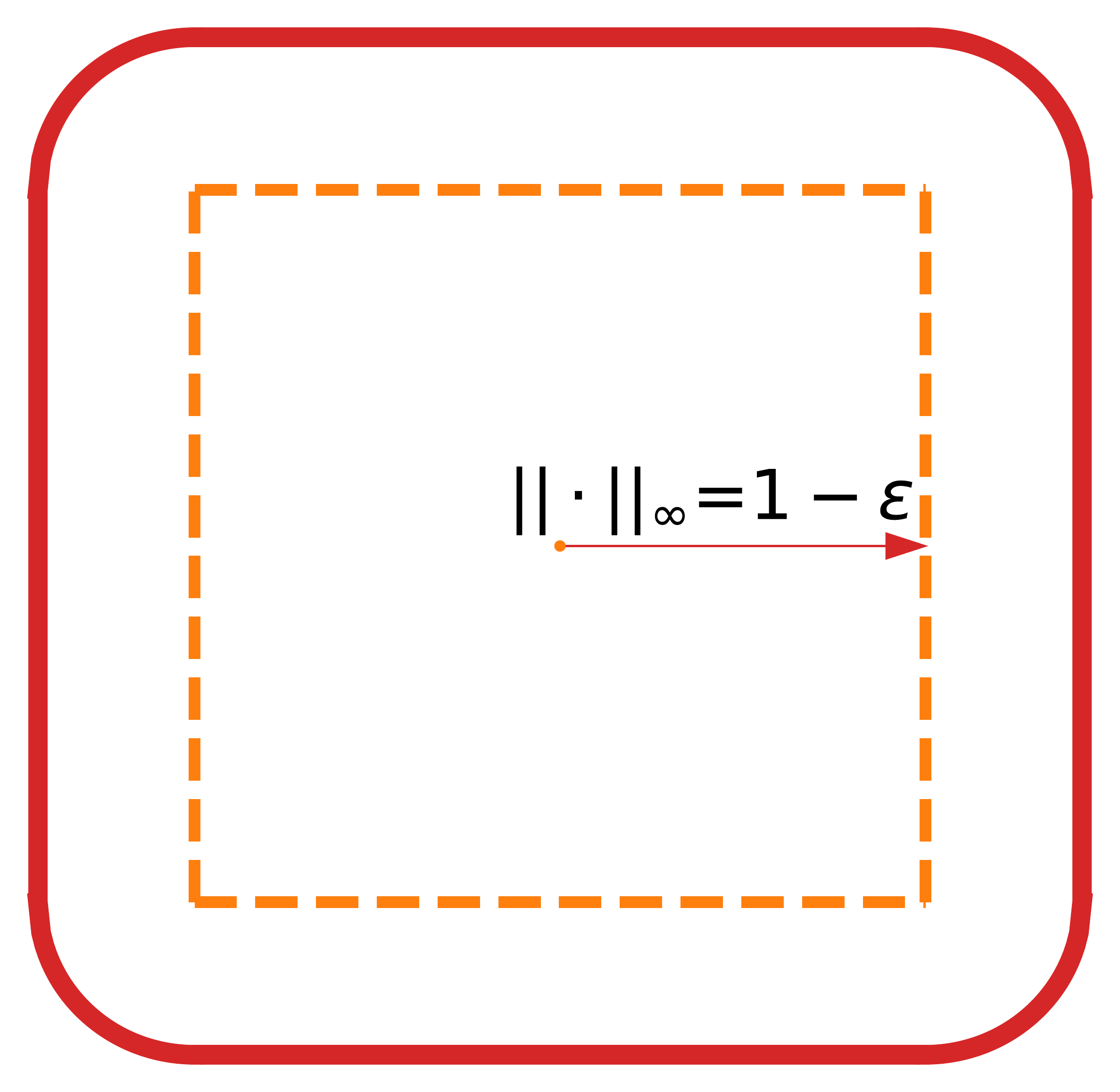}
    \caption{$\epsilon=0.3,q=2$}  
  \end{subfigure}
  \begin{subfigure}{\sublength\textwidth}
    \centering
    \includegraphics[width=1\linewidth]{q=2_epsilon=5}
    \caption{$\epsilon=0.5,q=2$}  
  \end{subfigure}
  \begin{subfigure}{\sublength\textwidth}
    \centering
    \includegraphics[width=1\linewidth]{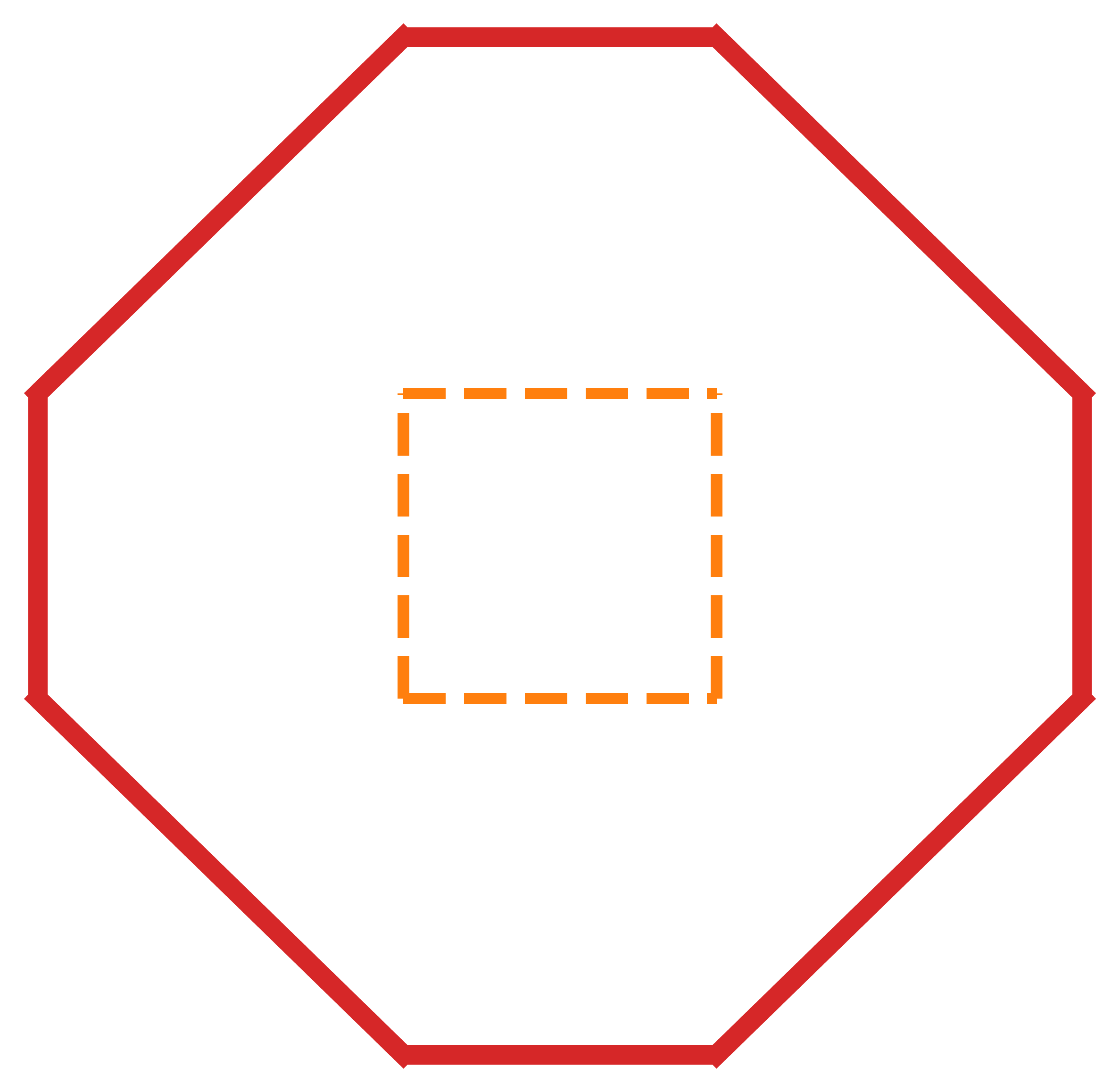}
    \caption{$\epsilon=0.7,q=1$}  
  \end{subfigure}  
  \begin{subfigure}{\sublength\textwidth}
    \centering
    \includegraphics[width=1\linewidth]{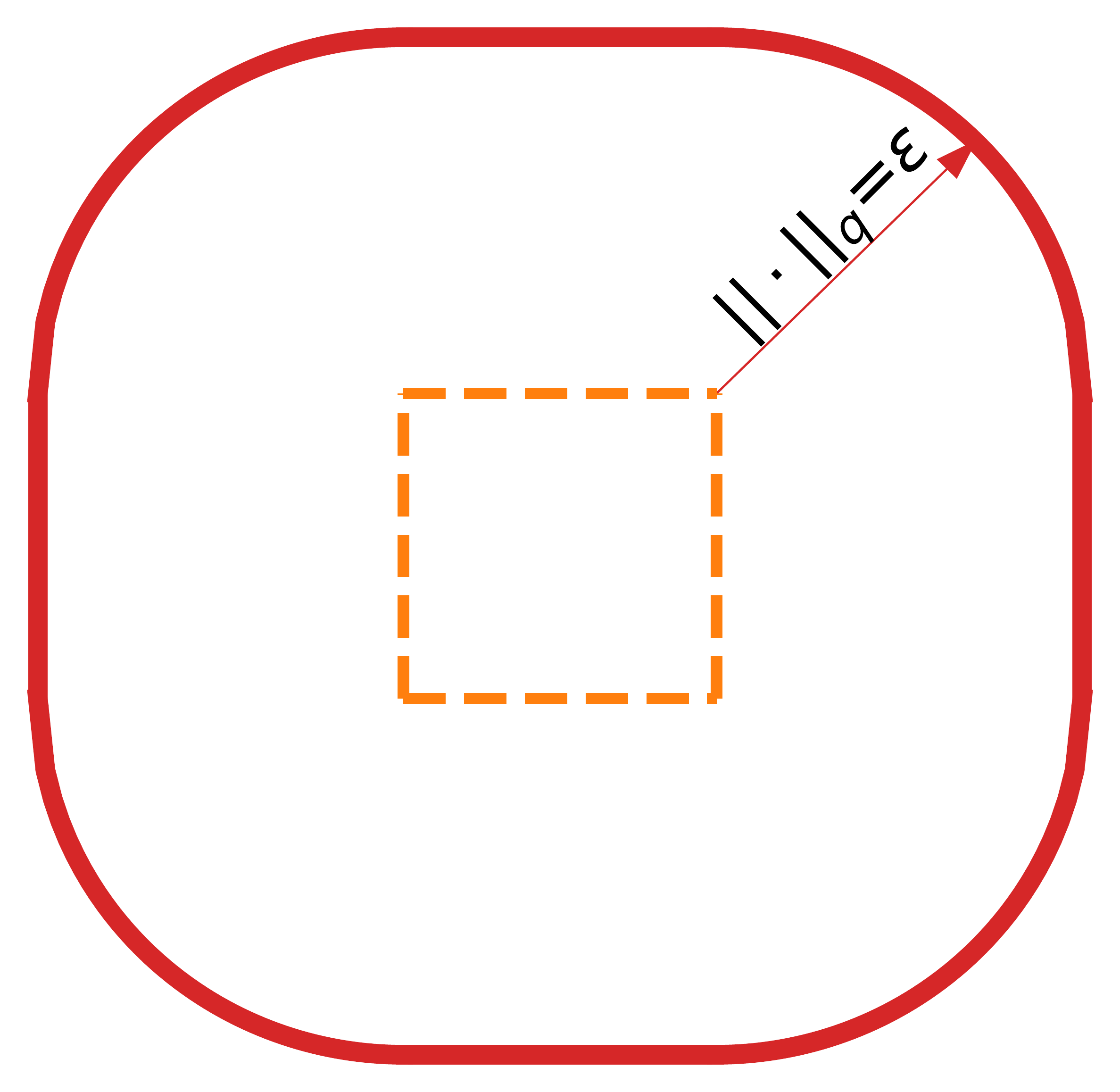}
    \caption{$\epsilon=0.7,q=2$}  
  \end{subfigure}  

  \caption{Unit ball in $\epsilon q$-norm for various choices of $(\epsilon,q)$. Each set is a Minkowski sum of $\ell_q$ ball of radius $\epsilon$ and $\ell_\infty$ ball of radius $1-\epsilon$, i.e. $\{x\in\mathbb{R}^p|\|x\|_{\epsilon q}\leq 1\}=\{a+b|\|a\|_q\leq \varepsilon,\|b\|_\infty\leq 1-\varepsilon \}$.}
\label{eqillustration}
\end{figure}

\begin{lemma}\label{lemma_dual_norm_primitive}
For any $\epsilon\in (0,1]$ and $q\in [1,\infty)$, the norm dual of the $\epsilon q$-norm is given by 
\begin{alignat}{2}
\|y\|_{\epsilon q}^*=\begin{cases}\epsilon\|y\|_{q/(q-1)}+(1-\epsilon)\|y\|_{1},\quad &q<\infty,\\
\|y\|_{1}, &q=\infty\end{cases}.
\end{alignat}
\end{lemma}
From Lemma \ref{lemma_dual_norm_primitive} we shall easily get two special cases: (1) $\|y\|_{\epsilon 2}^*=\epsilon\|y\|_2+(1-\epsilon)\|y\|_1$; (2) $\|y\|_{\epsilon 1}^*=\epsilon\|y\|_\infty+(1-\epsilon)\|y\|_1$. It is worth noting that every finite-dimensional normed space is reflexive, implying the dual of the dual norm is the norm itself. It also holds true for $\|\cdot\|_\ds$ and its dual in the lemma below.
\begin{lemma}\label{lemma_dual_norm}
First denote $\epsilon_g:=\frac{(1-\tau)w_g}{\tau+(1-\tau)w_g},\forall g\in\{1,...,G\}$ and $\alpha^*=(\alpha^*_1,...,\alpha^*_G)$ where $\alpha^*_g=\frac{\alpha_g}{\alpha_g-1}$. If $\alpha_g=\infty,\alpha_g^*=1$. The double sparsity norm $\|\cdot\|_\ds$ satisfies the following properties: $\forall \beta\in\mathbb{R}^p,x\in \mathbb{R}^p$,
\begin{align}
\|\beta\|_\ds =\sum\limits_{g\in \{1,...,G\}}(\tau+(1-\tau)w_g)\|\beta_{(g)}\|_{\epsilon_g \alpha_g^*}^*, \quad \textup{and}\quad  \|x\|_\ds^* = \max\limits_{g\in \{1,...,G\}}\frac{\|x_{(g)}\|_{\epsilon_g \alpha_g^*}}{\tau+(1-\tau)w_g}.
\end{align}
\end{lemma}
With the dual expression, one can follow the standard Lagrangian multiplier method to derive the dual for the primal problem. In our case, the dual program for the generalized double sparsity LASSO (\ref{DS-LASSO}) is given in the following lemma. 
\begin{lemma}\label{lemma_dual_dslasso}
The dual formulation of optimization problem \ref{DS-LASSO} is given by a convex optimization below
\begin{equation}
\begin{aligned}\label{dual_dslasso}
& \underset{\theta,u_g,v_g}{\text{max}}
& &  \|y\|_2^2-\|\frac{\lambda \theta}{2}-y\|_2^2\\
& & & u_g+v_g = X_{(g)}^\top \theta,\forall g=1,...,G,\\
& & &\|u_g\|_{\alpha_g^*} \leq\epsilon_g(\tau+(1-\tau)w_g),\forall g=1,...,G,\\
& & & \|v_g\|_{\infty} \leq(1-\epsilon_g)\big(\tau+(1-\tau)w_g\big),\forall g=1,...,G.
\end{aligned}
\end{equation}

\end{lemma}

\section{Convergence rates for exact sparsity}
We now shall state our main results (Theorem \ref{thm_l2error}) on the convergence rate of the estimator from problem (\ref{DS-LASSO}) when the unknown regressors coefficients $\beta^*$ is $(s,s_G)$-sparse. It  states the upper bound of convergence rate for L2 norm of the error as well as the lower bound of the regularizer required for the convergence. The proof of convergence needs some assumptions and required properties of the loss function and choice of penalty level $\lambda$ determined by the design matrix $X$ as well as the error distribution. Theorem \ref{thm_rec} gives a lower bound for $\|Xv\|_2$, also known as restricted eigenvalue convexity condition as seen in \cite{raskuttiRestrictedEigenvalueProperties} and \cite{geerConditionsUsedProve2009}. The condition is crucial in connecting the difference of the objective function and the estimator error.

We consider a broad class of random Gaussian designs. In particular, suppose the linear model $y=X\beta+\varepsilon$ in which each sample (i.e. each row of X) is from Gaussian distribution $\mathcal{N}(0,\Sigma)$.
\begin{theorem}\label{thm_rec}
For any Gaussian random design matrix $X\in\mathbb{R}^{n\times p}$ with i.i.d. $\mathcal{N}(0,\Sigma)$ rows, the following inequalities hold for all $v\in \mathbb{R}^p$ with probability $1-c\exp(-c'n)$, in which $c,c'$ are some constants,
\begin{align}
\|Xv\|_2&\geq \frac{\sqrt{n}}{4}\|\Sigma^{1/2}v\|_{2}-3\kappa_1\|v\|_\ds,\\
\|Xv\|_2&\geq \Big(\frac{\sqrt{n}}{4}\lambda_{\min}(\Sigma^{1/2})-3\kappa_1\kappa_2\Big)\|v\|_2,\quad \textup{ for } n> \frac{144\kappa_1^2\kappa_2^2}{\lambda_{\min}(\Sigma)}\label{rsc},
\end{align}
where 
\begin{align*}
\kappa_1:&=\max\limits_{g\in\{1,...,G\}} \Bigg[\frac{p_{g}^{(1/2-1/\alpha_g)_+}\sqrt{\Tr{(\Sigma_{(g)(g)})}}}{\tau+(1-\tau)w_g}\bigwedge \frac{3\sqrt{\log{p_g}\cdot\max_{i\in (g)}\Sigma_{ii}}}{(\tau+(1-\tau)w_g)(1-\epsilon_g+\epsilon_g p_g^{1/\alpha_g-1})}\Bigg],\\
\kappa_2:&= \sqrt{G}\max\limits_{g\in\{1,...,G\}}(\tau p_g^{1/2}+(1-\tau)w_g p_g^{(1/\alpha_g-1/2)_+}),\ \epsilon_g:=\frac{(1-\tau)w_g}{\tau+(1-\tau)w_g},\forall g\in\{1,...,G\}
\end{align*}

The second inequality (\ref{rsc}) is called restricted eigenvalues condition\cite{raskuttiRestrictedEigenvalueProperties}. One may construct various forms of $\ref{rsc}$, such as within a cone or star-shape containing a neighbor of zero(see \cite{negahbanUnifiedFrameworkHighDimensional2012a}).
\end{theorem}
The proof relied on Sudakov-Fernique inequality for Gaussian process comparison, Gaussian concentration inequality for Lipschitz functions and the peeling argument. Note that similar results for each element-wise or group-wise sparsity (replacing the $\|\cdot\|_\ds$ by the corresponding norm) have been obtained in past work on the basis pursuit and (group) LASSO (see \cite{bickelSimultaneousAnalysisLasso2009}, \cite{meinshausenLassotypeRecoverySparse2009}, \cite{negahbanUnifiedFrameworkHighDimensional2012a}, \cite{geerConditionsUsedProve2009}. Restricted condition of design matrix of general $\Sigma$-Gaussian ensemble can also be found in \cite{raskuttiMinimaxRatesEstimation2011}, \cite{raskuttiRestrictedEigenvalueProperties}. Rudelson and Zhou \cite{rudelsonReconstructionAnisotropicRandom2013} extended the ensemble class to the case of sub-Gaussian designs with substantial covariates dependencies. 

We further make two mild assumptions on design matrix $X$ and noisy observations respectively. 
\begin{assumption}\label{column_normalization}
We consider a column normalization assumption on design matrix $X$ based on the double sparsity norm, 
\begin{align}
\|X\|_{(g)\to 2}:=\sup\limits_{u:\epsilon_g\|u\|_{\alpha_g}+(1-\epsilon_g)\|u\|_{1}\leq 1}\|X_{(g)}u\|_2=\sqrt{n}
\end{align}
\end{assumption}
\begin{remark}
It is worth noting that is a natural extension of the column normalization condition $\frac{\|X_{j}\|_2}{\sqrt{n}}, j\in\{1,...,p\}$ as seen in \cite{negahbanUnifiedFrameworkHighDimensional2012a} which is no loss of generality as in practice linear model can be scaled appropriately so that the condition can be satisfied.
\end{remark}
\begin{assumption}\label{sub_gaussian}
We consider the observation noise/error $\varepsilon\in \mathbb{R}^n$ is zero-mean and has sub-Gaussian tails, i.e. 
\begin{align}
P(\frac{\varepsilon^\top u}{\|u\|_2}\geq\delta)\leq 2\exp(-\frac{n\delta^2}{2\sigma^2}), \forall \delta>0
\end{align}
for some constant $\sigma>0$.
\end{assumption}
\begin{remark}
Many probability distributions can satisfies the above assumption. In particular, the assumption holds if the noise distribution is standardized normal distribution or bounded random variables. As a matter of fact, any Gaussian distribution will satisfy the sub-Gaussian tail properties. 
\end{remark}
Theorem \ref{thm_lambda} below gives a high probability upper bound for $\|X^\top\varepsilon\|_\ds^*$ which in turn provides a guideline for the penalty level $\lambda$ in recovering sparsity of estimators.  
\begin{theorem}\label{thm_lambda}
Suppose that $X$ satisfies the block column normalization condition (Assumption \ref{column_normalization}) and the observation noise $\varepsilon$ is sub-Gaussian, satisfying sub-Gaussian tail (Assumption \ref{sub_gaussian}). Then we have with probability at least $1-\frac{2}{G^2}$, the following inequality holds
\begin{align}
\frac{\|X^\top \varepsilon\|_\ds^*}{n}\leq \max\limits_{g\in\{1,...,G\}}\sigma\frac{\bigg[\frac{\epsilon_gp_g^{(1/\alpha_g-1/2)_+}+(1-\epsilon_g)p_g^{1/2}}{\sqrt{n}}\Big(\big[p_{g}^{(1/2-1/\alpha_g)_+}\sqrt{p_g}\big] \bigwedge \frac{{p_g}^{1/\alpha_g^*}\sqrt{2\log{p_g}}}{{p_g}^{1/\alpha_g^*}(1-\epsilon_g)+\epsilon_g} \Big)+  \sqrt{\frac{6\log{G}}{n}}\bigg]}{\tau+(1-\tau)w_g}.
\end{align}
\end{theorem}
With the above two theorems as well as these conditions, we shall be able to give the convergence in the following novel result:
\begin{theorem}\label{thm_l2error}
Suppose that the design matrix satisfies the column normalization condition (Assumption \ref{column_normalization}) and the restricted eigenvalues condition (\ref{rsc}). Moreover, the noise $\varepsilon$ is sub-Gaussian (Assumption \ref{sub_gaussian}). The generalized double sparsity estimator $\hat{\beta}$ with 
\begin{align}\label{eqn_lambda}
\lambda\geq2\max\limits_{g\in\{1,...,G\}}\sigma\frac{\bigg[\Big(\epsilon_gp_g^{(1/\alpha_g-1/2)_+}+(1-\epsilon_g)p_g^{1/2}\Big)\Big(\big[p_{g}^{(1/2-1/\alpha_g)_+}\sqrt{np_g}\big] \bigwedge \frac{{p_g}^{1/\alpha_g^*}\sqrt{2n\log{p_g}}}{{p_g}^{1/\alpha_g^*}(1-\epsilon_g)+\epsilon_g} \Big)+  \sqrt{6n\log{G}}\bigg]}{\tau+(1-\tau)w_g}
\end{align}
satisfies the following $L_2$ error with probability at least $1-\frac{2}{G^2}-ce^{-c'/n}$ for some positive constants $c,c'$,
\begin{align}
\|\hat{\beta}-\beta^*\|_2^2\leq \frac{4\lambda^2\Big[\tau \sqrt{s}+(1-\tau) \sqrt{s_G}\max\limits_{g\in\{1,...,G\}}[w_gp_g^{(1/\alpha-1/2)_+}]\Big]^2}{\Big(\frac{\sqrt{n}}{4}\lambda_{\min}(\Sigma^{1/2})-3\kappa_1\kappa_2\Big)^4},\label{eqn_l2error}
\end{align}
where 
\begin{align*}
\kappa_1&=\max\limits_{g\in\{1,...,G\}} \Bigg[\frac{p_{g}^{(1/2-1/\alpha_g)_+}\sqrt{\Tr{(\Sigma_{(g)(g)})}}}{\tau+(1-\tau)w_g}\bigwedge \frac{3\sqrt{\log{p_g}\cdot\max_{i\in (g)}\Sigma_{ii}}}{(\tau+(1-\tau)w_g)(1-\epsilon_g+\epsilon_g p_g^{1/\alpha_g-1})}\Bigg],\\
\kappa_2&= \sqrt{G}\max\limits_{g\in\{1,...,G\}}(\tau p_g^{1/2}+(1-\tau)w_g p_g^{(1/\alpha_g-1/2)_+}),\\ 
\epsilon_g&=\frac{(1-\tau)w_g}{\tau+(1-\tau)w_g},\forall g\in\{1,...,G\},\ n> \frac{144\kappa_1^2\kappa_2^2}{\lambda_{\min}(\Sigma)}.
\end{align*}
\end{theorem}
\begin{remark}
As the results are general and can be applied to any $\{(\alpha_g,w_g,p_g)\}_{g=1}^G\in [1,\infty)^G\times (0,\infty)^G\times \mathbb{N}_+^G$ and $\tau\in[0,1]$, we can observes how different choices of regularization and sparsity norm will affect the convergence rates. We consider the following cases. Case (1),(2),(3),(4) are single sparsity cases (either element-wise or group-wise regularization but not both.) Case (5),(6),(7) are double sparsity where simultaneous sparsity is considered. 
\end{remark}
\paragraph{(1). $\bm{\tau=1}$:} The case $\tau=1$ corresponds to the traditional LASSO problem with element-wise sparsity. Accordingly, $\epsilon_g=0,\forall g\in\{1,...,G\}$ in this case. As there is no group norm in the regularization term or group structures are utilized in the objective function, one can assume $p_g=1,G=p$. Expression (\ref{eqn_lambda}), $\kappa_1$ and $\kappa_2$ becomes 
\begin{align}
\lambda&\geq2\max\limits_{g\in\{1,...,G\}}\sigma\bigg[\sqrt{n} \wedge \sqrt{2n\log{p_g}} +  \sqrt{6n\log{G}}\bigg]=2\max\limits_{g\in\{1,...,p\}}\sigma\sqrt{6n\log{p}}=2\sigma\sqrt{6n\log{p}}\\
\kappa_1&=\max\limits_{g\in\{1,...,G\}} \Bigg[\sqrt{\Tr{(\Sigma_{(g)(g)})}}\bigwedge \frac{3\sqrt{\log{p_g}\cdot\max_{i\in (g)}\Sigma_{ii}}}{(1-\epsilon_g+\epsilon_g)}\Bigg]=0,\ \kappa_2= \sqrt{G}\max\limits_{g\in\{1,...,G\}}\tau p_g^{1/2}=\sqrt{p}.
\end{align}
By letting $\lambda = \sigma\sqrt{6n\log{p}}$, Expression (\ref{eqn_l2error}) then becomes
\begin{align}
\|\hat{\beta}-\beta^*\|_2^2\leq \frac{1024\lambda^2s}{n^2\lambda_{\min}(\Sigma^{2})}=\frac{6144}{\lambda_{\min}(\Sigma^{2})}\sigma^2\frac{s\log{p}}{n}
\end{align}
The regularization parameter $\lambda=O(\sigma\sqrt{n\log{p}})$ as well as the L2 norm of error $\|\hat{\beta}-\beta^*\|_2=O(\sigma\sqrt{\frac{s\log{p}}{n}})$ are an exact recovery to well-established past work (see \cite{bickelSimultaneousAnalysisLasso2009},\cite{candesDantzigSelectorStatistical2007},\cite{meinshausenLassotypeRecoverySparse2009}). Our proof illustrates a novel approach which generalizes traditional LASSO problem easily. 

\paragraph{(2). $\bm{\tau}=\bm{0},\bm{\alpha}\bm{\equiv}\bm{1}$:} This case share much similarity as the previous one. They both only contain $\ell_1$ regularization. The key difference is that in the traditional LASSO, element-wise regularization term has equal weight while in this case different groups have different weights determined by $w_g$. It is easy to verify that $\epsilon_g=1,\forall g\in\{1,...,G\}$. Expression (\ref{eqn_lambda}), $\kappa_1$ and $\kappa_2$ becomes 
\begin{align}
\lambda&\geq2\sigma\sqrt{n}\max\limits_{g\in\{1,...,G\}}\frac{ p_g \bigwedge \sqrt{2p_g\log{p_g}}+  \sqrt{6\log{G}}}{w_g}=2\sigma\sqrt{n}\max\limits_{g\in\{1,...,G\}}\frac{\sqrt{2p_g\log{p_g}}+  \sqrt{6\log{G}}}{w_g}\\
\kappa_1&=\max\limits_{g\in\{1,...,G\}} \frac{1}{w_g} \sqrt{\Tr{(\Sigma_{(g)(g)})}} \bigwedge 3\sqrt{\log{p_g}\cdot\max_{i\in (g)}\Sigma_{ii}},\kappa_2= \sqrt{G}\max\limits_{g\in\{1,...,G\}}w_g p_g^{1/2}
\end{align}
By letting $\lambda =2\sigma\sqrt{n}\max\limits_{g\in\{1,...,G\}}\frac{\sqrt{2p_g\log{p_g}}+  \sqrt{6\log{G}}}{w_g}$, Expression (\ref{eqn_l2error}) then becomes
\begin{align}
\|\hat{\beta}-\beta^*\|_2\leq \sqrt{\frac{4\lambda^2 s_G\max\limits_{g\in\{1,...,G\}}w_g^2p_g}{\Big(\frac{\sqrt{n}}{4}\lambda_{\min}(\Sigma^{1/2})-3\kappa_1\kappa_2\Big)^4}}\lesssim \sigma\sqrt{s_G}\Big(\sqrt{\frac{\max\limits_{g=1,...,G}p_g^2\log{p_g}}{n}}+\sqrt{\frac{p_g\log{G}}{n}}\Big).
\end{align}
By looking into the term $\max\limits_{g\in\{1,...,G\}}w_g^2p_g$, one may argue that to obtain a smaller factor for the L2 error term for a better performance on sparsity recovery, following a minmax argument and assigning small weight on groups with large covariates size in such way to minimize $\max\limits_{g\in\{1,...,G\}}w_g^2p_g$ could be an option.
\paragraph{(3). $\bm{\tau}=\bm{0},\bm{\alpha}\bm{\equiv}\bm{2}$:} The case $(\tau,\alpha)\equiv (0,2)$ corresponds to the traditional group LASSO problem with only group-wise sparsity. Expression (\ref{eqn_lambda}) and Expression (\ref{eqn_l2error}) become
\begin{align}
\lambda&\geq2\max\limits_{g\in\{1,...,G\}}\sigma\frac{\sqrt{np_g} \bigwedge \sqrt{2np_g\log{p_g}} +  \sqrt{6n\log{G}}}{w_g}=2\max\limits_{g\in\{1,...,G\}}\sigma\frac{\sqrt{np_g}+  \sqrt{6n\log{G}}}{w_g}\\
\|\hat{\beta}-\beta^*\|_2&\leq \sqrt{\frac{4\lambda^2 s_G\max\limits_{g\in\{1,...,G\}}w_g^2}{\Big(\frac{\sqrt{n}}{4}\lambda_{\min}(\Sigma^{1/2})-3\kappa_1\kappa_2\Big)^4}}\lesssim \sigma \sqrt{s_G} \Big(\sqrt{\frac{\max\limits_{g=1,...,G}p_g}{n}}+\sqrt{\frac{\log{G}}{n}}\Big). 
\end{align}
The regularization parameter $\lambda=O\big(\sigma\sqrt{n\max\limits_{g=1,...,G}p_g}+\sigma\sqrt{n\log{G}}\big)$ as well as the L2 norm of error $\|\hat{\beta}-\beta^*\|_2=O\bigg(\sigma \sqrt{s_G\frac{\max\limits_{g=1,...,G}p_g}{n}}+\sigma \sqrt{s_G\frac{\log{G}}{n}}\Big)$ match with the past results derived in \cite{huangBenefitGroupSparsity2010}, \cite{louniciTakingAdvantageSparsity2009}, \cite{negahbanUnifiedFrameworkHighDimensional2012a} for exact block sparsity. Different from previous results, our finding show the influence of various $w_g,p_g$ on convergence rate. To be more specific, if one of the $w_g$ is significantly larger than other weights, $\lambda^2\max_{g=1,...,G}w_g^2\asymp \sigma^2n\log{G}+\sigma^2 n\max\limits_{g=1,...,G}p_g$. The L2 norm of error is therefore bounded by a multiplier determined by $p_g$. 

\paragraph{(4). $\bm{\tau}=\bm{0},\bm{\alpha}\bm{\equiv}\bm{\infty}$:} The case $(\tau,\alpha)\equiv (0,\infty)$ corresponds to the group sparsity problem equipped with $\ell_1/\ell_\infty$-norm. Similar formulation in the optimization community as seen in \cite{turlachSimultaneousVariableSelection2005} considers the $\ell_\infty$-norm in the constraints and discuss the resulting convex quadratic program thereafter. In our setting, Expression (\ref{eqn_lambda}) and Expression (\ref{eqn_l2error}) become
\begin{align}
\lambda&\geq2\sigma\max\limits_{g\in\{1,...,G\}}\frac{1}{w_g}\sqrt{n}\Big(p_{g} \bigwedge p_{g}\sqrt{2\log{p_g}} \Big)+  \sqrt{6n\log{G}}=2\sigma\sqrt{n}\max\limits_{g\in\{1,...,G\}}\frac{1}{w_g}(p_{g} +\sqrt{6\log{G}})\\
\|\hat{\beta}-\beta^*\|_2&\leq \sqrt{\frac{4\lambda^2s_G\max\limits_{g\in\{1,...,G\}}w_g^2}{\Big(\frac{\sqrt{n}}{4}\lambda_{\min}(\Sigma^{1/2})-3\kappa_1\kappa_2\Big)^4}}\lesssim \sigma \sqrt{s_G} \Big(\sqrt{\frac{\max\limits_{g=1,...,G}p_g^2}{n}}+\sqrt{\frac{\log{G}}{n}}\Big). 
\end{align}
The L2 norm of error $\|\hat{\beta}-\beta^*\|_2=O\bigg(\sigma \sqrt{s_G}\sqrt{\frac{\max\limits_{g=1,...,G}p_g^2}{n}}+\sigma \sqrt{s_G}\sqrt{\frac{\log{G}}{n}}\Big)$ has similar components as that in $(\tau,\alpha)\equiv (0,\infty)$, having a estimation term ($\frac{\sqrt{s_G}p_g}{\sqrt{n}}$) and a search term ($\sqrt{s_G\frac{\log{G}}{m}}$). Notice that we recover the additional factor of $p_g$ in the estimation term as discussed in \cite{negahbanUnifiedFrameworkHighDimensional2012a}.

Now we consider the following three double sparsity cases $\alpha\equiv 1,2,\infty$ with $\tau\in(0,1)$.

\paragraph{(5). $\bm{\tau}\bm{\in(0,1)},\bm{\alpha}\bm{\equiv}\bm{1}$:} With a little derivation, it is easy to find out that it reduces to Case (2). We may also follow the main Theorem for a sanity check. Expression (\ref{eqn_lambda}) and Expression (\ref{eqn_l2error}) become
\begin{align}
\lambda&\geq2\sigma\sqrt{n}\max\limits_{g\in\{1,...,G\}}\frac{p_g\bigwedge\sqrt{2p_g\log{p_g}} +  \sqrt{6\log{G}}}{\tau+(1-\tau)w_g}=2\sigma\sqrt{n}\max\limits_{g\in\{1,...,G\}}\frac{\sqrt{2p_g\log{p_g}} +  \sqrt{6\log{G}}}{\tau+(1-\tau)w_g}\\
\|\hat{\beta}-\beta^*\|_2&\lesssim \sigma\big(\sqrt{\frac{\max_gp_g\log{p_g}}{n}}+\sqrt{\frac{\log{G}}{n}}\big)\lambda\Big[\tau \sqrt{s}+(1-\tau) \sqrt{s_G}\max\limits_{g\in\{1,...,G\}}[w_g\sqrt{p_g}]\Big]
\end{align}
\paragraph{(6). $\bm{\tau}\bm{\in(0,1)},\bm{\alpha}\bm{\equiv}\bm{2}$:} The case is known as sparse-group LASSO problem as discussed in \cite{idaFastSparseGroup2019}, \cite{ndiayeGAPSafeScreening2016},\cite{simonSparseGroupLasso2013},\cite{wangTwoLayerFeatureReduction2014}. In our setting, Expression (\ref{eqn_lambda}) and Expression (\ref{eqn_l2error}) become
\begin{align}
\lambda&\geq2\sigma\sqrt{n}\max\limits_{g\in\{1,...,G\}}\frac{\Big(\epsilon_g+(1-\epsilon_g)p_g^{1/2}\Big)\sqrt{p_g}+  \sqrt{6\log{G}}}{\tau+(1-\tau)w_g}\\
\|\hat{\beta}-\beta^*\|_2&\lesssim \frac{\sigma}{\sqrt{n}}\big[\max_g(\varepsilon_g+(1-\varepsilon_g)\sqrt{p_g})+\sqrt{\log{G}}\big]\cdot\Big[\tau \sqrt{s}+(1-\tau) \sqrt{s_G}\max\limits_{g\}}w_g\big]
\end{align}
The above upper bound and choice of $\lambda$ recover similar results for equal weight $w_g\equiv w$ case as derived in \cite{caiSparseGroupLasso2019}.
\paragraph{(7). $\bm{\tau}\bm{\in(0,1)},\bm{\alpha}\bm{\equiv}\bm{\infty}$:} The case corresponds to the double group sparsity problem equipped with $\ell_1/\ell_\infty$-norm. In our setting, Expression (\ref{eqn_lambda}) and Expression (\ref{eqn_l2error}) become  
\begin{align}
\lambda&\geq2\sigma\sqrt{n}\max\limits_{g\in\{1,...,G\}}\frac{\Big(\epsilon_g+(1-\epsilon_g)p_g^{1/2}\Big)p_{g}+  \sqrt{6\log{G}}}{\tau+(1-\tau)w_g}\\
\|\hat{\beta}-\beta^*\|_2&\lesssim \frac{\sigma}{\sqrt{n}}\Big[\max\limits_{g}\big(\epsilon_g+(1-\epsilon_g)p_g^{1/2}\big)p_{g}+  \sqrt{6\log{G}}\Big]\cdot\Big[\tau \sqrt{s}+(1-\tau) \sqrt{s_G}\max\limits_{g\in\{1,...,G\}}w_g\Big].
\end{align}
\section{Conclusion}
In this paper we investigate the generalized group sparsity. we first a new vector norm in Euclidean space $\mathbb{R}^n$-$\epsilon q$-norm. We prove its validity as a vector norm and relationship with $\ell_q$ norm. Unique decomposition in term of $\epsilon q$-norm as well as the dual norm are derived, which naturally leads us to discuss the implication in the generalized double sparsity problem. The main convergence results give a guideline for penalty level and show a uniform convergence result in term of L2 norm of estimation error. The main theorem generalizes various convergence rates both for single regularization cases such as LASSO and group LASSO and for double regularization cases such as group-sparse LASSO. 
\medskip
\small
\bibliographystyle{plain}
\bibliography{pco}

\newpage
\section*{Supplement}
\begin{lemma*}
For any $\epsilon\in (0,1]$ and $q\geq 1$, the unique nonnegative solution $\nu\in \mathbb{R}^+$ of equation (\ref{defepsq}) defines a vector norm in $\mathbb{R}^p$.
\end{lemma*}

\begin{proof}[Proof of Lemma \ref{lemma_isnorm}]
If $q = \infty$, $\|\cdot\|_{\epsilon q}=\|\cdot\|_{\infty}$.
If $q<\infty, \epsilon=1$, we have $\nu=\|x\|_q$.
If $q<\infty, \epsilon<1$, we denote $h(v,x)=\|S_{(1-\epsilon) v}(x)\|_q-\epsilon v$. As $h(v,x)$ is monotonically decreasing and continuous in $v$. In addition, $h(0,x)=\|x\|_q\geq 0$ and $h(v,x)\equiv -\epsilon v<0$ for $v$ sufficiently large (i.e. $v>\frac{1}{1-\epsilon}\|x\|_\infty$). Hence $h(v,x)$ has a unique and nonnegative root for any fixed $x$. We denote $v(x)$ as the non-negative root of $h(v,x)$. 

For any $\alpha\in \mathbb{R}$, the solution $v=|\alpha|v(x)$ satisfies the equality $h(v,\alpha x)=0$. Therefore we have $v(\alpha x)=|\alpha|v(x).$ Moreover, it is easy to verify that $v(x)=0 \iff x=0$.

Now we prove $v(x+y)\leq v(x)+v(y)$.

It is easy to see that 
\begin{equation}
\begin{aligned}\label{derivation_triinequality}
h(v(x)+v(y),x+y)+\epsilon (v(x)+v(y))&=\Big\|\big[|x+y|-(1-\epsilon)(v(x)+v(y))\big]_+\Big\|_q\\
&\leq \Big\|\big[|x|+|y|-(1-\epsilon)(v(x)+v(y))\big]_+\Big\|_q\\
&\leq \|[|x|-(1-\epsilon)v(x)]_++[|y|-(1-\epsilon)v(y)]_+\|_q\\
&\leq \|[|x|-(1-\epsilon)v(x)]_+\|_q+\|[|y|-(1-\epsilon)v(y)]_+\|_q\\
&= h(v(x),x)+\epsilon v(x)+h(v(y),y)+\epsilon v(y)\\
&= \epsilon v(x)+\epsilon v(y)
\end{aligned}
\end{equation}
where we are using the fact that $(a+b)_+\leq (a)_++(b)_+$ and triangle inequality for $\ell_1$ and $\ell_q$-norm.
Therefore we have 
\begin{align}
h(v(x)+v(y),x+y)\leq 0 \ \implies v(x+y)\leq v(x)+v(y)
\end{align}
Notice that for $x=0$ or $y=0$, the triangle inequality trivially holds. For $x,y\neq 0$, denote 
$$I(x)=\{i||x_i|>(1-\epsilon)v(x)\}.$$
It is easy to verify that the sufficient and necessary condition for 
\begin{align}\label{tri_equality}
v(x+y)=v(x)+v(y)
\end{align}
is
\begin{equation}
\begin{aligned}\label{tri_iffcondition1}
I(x)&=I(y)\\
\exists \alpha>0,\quad  y_i &= \alpha x_i,\forall i \in I(x).
\end{aligned}
\end{equation}
Expression (\ref{tri_iffcondition1}) imply Equality (\ref{tri_equality}) trivially holds. Suppose Equality (\ref{tri_equality}) is true, the inequalities in (\ref{derivation_triinequality}) holds as equalities. In particular, the last inequality uses the triangle inequality for $\ell_q$-norm. Therefore, we should have $\exists \alpha>0$,
\begin{align}
\big[|y|-(1-\epsilon)v(y)\big]_+=\alpha \big[|x|-(1-\epsilon)v(x)\big]_+.
\end{align}
Then we must have 
\begin{equation*}
\begin{aligned}
I(x)&=I(y),\\
v(y)&=\alpha v(x)\ &&\implies |y_i|&&=\alpha |x_i|,&&\forall i\in I(x),\\
|x_i+y_i|&=|x_i|+|y_i| ,\forall i\in I(x) \ &&\implies y_i&&=\alpha x_i,&&\forall i\in I(x).
\end{aligned}
\end{equation*}
\end{proof}
\begin{lemma*}
For any $\epsilon\in (0,1]$ and $q\geq 1$, the following tight bounds of the $\epsilon q$-norm holds for any $x\in \mathbb{R}^p$,
\begin{alignat}{2}
\frac{\|x\|_1}{p^{1-1/q}(p^{1/q}(1-\epsilon)+\epsilon)}&\leq \|x\|_{\epsilon q}&&\leq \|x\|_1\\
\frac{\|x\|_q}{p^{1/q}(1-\epsilon)+\epsilon}&\leq \|x\|_{\epsilon q}&&\leq \|x\|_q\\
\frac{\|x\|_2}{p^{(1/2-1/q)_+}(p^{1/q}(1-\epsilon)+\epsilon)} &\leq \|x\|_{\epsilon q}&&\leq p^{(1/q-1/2)_+}\|x\|_2\\
\|x\|_\infty&\leq \|x\|_{\epsilon q}&&\leq \frac{p^{1/q}\|x\|_\infty}{p^{1/q}(1-\epsilon)+\epsilon}
\end{alignat}
\end{lemma*}
\begin{proof}[Proof of Lemma \ref{lemma_tightbound}]
Notice that 
\begin{align}
h(v,x)+\epsilon v=\|[|x|-(1-\epsilon)v]_+\|_q=\|[\epsilon|x|-(1-\epsilon)(v-|x|)]_+\|_q
\end{align}
If $v=\|x\|_q$, we have $v-|x|\geq v-\|x\|_\infty \geq 0$, implying $\|[\epsilon|x|-(1-\epsilon)(\|x\|_q-|x|)]_+\|_q\leq \epsilon\|x\|_q$. Therefore, we have $h(\|x\|_q,x)\leq 0$, implying that $\|x\|_{\epsilon q}\leq \|x\|_q$. Since $\|x\|_q\leq \|x\|_1$. We also have $\|x\|_{\epsilon q}\leq \|x\|_1$. Notice that if $x=e_i$ where $e_i$ is a unit vector that $i$th entry is one and the others are zero, we have $\|x\|_{\epsilon q}=\|x\|_q=\|x\|_1$.

Notice that 
\begin{align}
h(\|x\|_\infty,x)+\epsilon \|x\|_\infty &=\|[|x|-(1-\epsilon)\|x\|_\infty]_+\|_q\geq \|[|x|-(1-\epsilon)\|x\|_\infty]_+\|_\infty\\
&=\|x\|_\infty-(1-\epsilon)\|x\|_\infty=\epsilon \|x\|_\infty.
\end{align}
Therefore $h(\|x\|_\infty,x)\geq 0$, implying $\|x\|_\infty\leq \|x\|_{\epsilon q}$. The inequality is achievable, for example if $x=e_i$.

Denote $v_\infty = \frac{p^{1/q}\|x\|_\infty}{p^{1/q}(1-\epsilon)+\epsilon}$ and we have 
\begin{align}
h(v_\infty,x)+\epsilon v_\infty&=\|[|x|-(1-\epsilon)v_\infty]_+\|_q\leq p^{1/q}\|[|x|-(1-\epsilon)v_\infty]_+\|_\infty\\
&=\frac{p^{1/q}}{p^{1/q}(1-\epsilon)+\epsilon}\|[\epsilon |x|-(1-\epsilon)(\|x\|_\infty-|x|)]_+\|_\infty\\
&\leq \frac{p^{1/q}}{p^{1/q}(1-\epsilon)+\epsilon}\epsilon\|x\|_\infty=\epsilon v_\infty
\end{align}
Therefore $h(v_\infty,x)\leq 0$, implying $\|x\|_{\epsilon q}\leq \frac{p^{1/q}\|x\|_\infty}{p^{1/q}(1-\epsilon)+\epsilon}$. The inequality is achievable, for example if $x$ is an all-one vector.

Finally, based on $\epsilon$-decomposition, 
\begin{align}
x = x^{(\epsilon,q)}+x^{(1-\epsilon,q)} 
\end{align}
with $x^{(\epsilon,q)},x^{(1-\epsilon,q)}\in \mathbb{R}^p$ such that 
\begin{align}
\|x^{(\epsilon,q)}\|_q&=\epsilon\|x\|_{\epsilon q},\\
\|x^{(1-\epsilon,q)}\|_\infty &= (1-\epsilon)\|x\|_{\epsilon q},
\end{align} 
implying
\begin{align}
\|x\|_{\epsilon q}&=\epsilon\|x\|_{\epsilon q}+(1-\epsilon)\|x\|_{\epsilon q}=\|x^{(\epsilon,q)}\|_q+\|x^{(1-\epsilon,q)}\|_\infty\\
\|x\|_{q}&\leq \|x^{(\epsilon,q)}\|_q+\|x^{(1-\epsilon,q)}\|_q\leq \epsilon\|x\|_{\epsilon q}+p^{1/q}\|x^{(1-\epsilon,q)}\|_\infty=\epsilon\|x\|_{\epsilon q}+p^{1/q}(1-\epsilon)\|x\|_{\epsilon q}.
\end{align}
Therefore we have $\frac{\|x\|_q}{p^{1/q}(1-\epsilon)+\epsilon}\leq \|x\|_{\epsilon q}$. Notice that $\|x\|_q\geq q^{1/q-1} \|x\|_1$, we then have $\frac{\|x\|_1}{q^{1-1/q}(p^{1/q}(1-\epsilon)+\epsilon)}\leq \|x\|_{\epsilon q}$. The inequalities are achievable, for example if $x$ is an all-one vector.

Since $\frac{\|x\|_2}{p^{(1/2-1/q)_+}}\leq \|x\|_q\leq p^{(1/q-1/2)_+}\|x\|_2$, we have 

\begin{align}
\frac{\|x\|_2}{p^{(1/2-1/q)_+}(p^{1/q}(1-\epsilon)+\epsilon)} \leq \frac{\|x\|_q}{p^{1/q}(1-\epsilon)+\epsilon}\leq \|x\|_{\epsilon q}\leq \|x\|_q\leq p^{(1/q-1/2)_+}\|x\|_2.
\end{align}
It concludes the proof.
\end{proof}
Denote 
\begin{align}
I_* &=\{i:|x_i|-(1-\epsilon)\|x\|_\epsilon>0\},\\
U(v)&=\{u\in \mathbb{R}^p:\|u\|_q\leq \epsilon v\},\\
V(v)&=\{v\in \mathbb{R}^p:\|v\|_\infty\leq (1-\epsilon) v\}.
\end{align}
\begin{lemma*}[$\varepsilon$-decomposition]
Any vector $x\in \mathbb{R}^p$ can be uniquely decomposed and written in the form 
\begin{align}
x = x^{(\epsilon,q)}+x^{(1-\epsilon,q)} 
\end{align}
with $x^{(\epsilon,q)},x^{(1-\epsilon,q)}\in \mathbb{R}^p$ such that 
\begin{align}
\|x^{(\epsilon,q)}\|_q&=\epsilon\|x\|_{\epsilon q},\\
\|x^{(1-\epsilon,q)}\|_\infty &= (1-\epsilon)\|x\|_{\epsilon q}.
\end{align} 
In addition, we have 
\begin{align}
\{x\in\mathbb{R}^p:\|x\|_{\epsilon q}\leq \nu\}= \{u+v:u,v\in\mathbb{R}^p,\|u\|_q\leq \epsilon \nu,\|v\|_\infty\leq (1-\epsilon) \nu\}. 
\end{align}
\end{lemma*}
\begin{proof}[Proof of Lemma \ref{lemma_decomposition}]
Denote \begin{align}
I_* &=\{i:|x_i|-(1-\epsilon)\|x\|_\epsilon>0\},\\
U(v)&=\{u\in \mathbb{R}^p:\|u\|_q\leq \epsilon v\},\\
V(v)&=\{v\in \mathbb{R}^p:\|v\|_\infty\leq (1-\epsilon) v\}.
\end{align}
We consider
\begin{align}
x^{(\epsilon,q)}&=\text{sgn}(x)\odot [|x|-(1-\epsilon)\|x\|_{\epsilon q}]_+,\\
x^{(1-\epsilon,q)}&=\text{sgn}(x)\odot \big\{|x|-[|x|-(1-\epsilon)\|x\|_{\epsilon q}]_+\big\}.
\end{align}
Then we have $x^{(\epsilon,q)}+x^{(1-\epsilon,q)}=x$, $\|x^{(\epsilon,q)}\|_q=\epsilon\|x\|_{\epsilon q},\|x^{(1-\epsilon,q)}\|_\infty = (1-\epsilon)\|x\|_{\epsilon q}.$

We now prove the uniqueness of this decomposition. If $x=0$,the statement is trivial. For $x\neq 0$, consider any $v\in \mathbb{R}^p$ such that $v\in V(\|x\|_{\epsilon q})$ and $v \neq x^{(1-\epsilon,q)}$. It is easy to see that 
\begin{align}
\|x-v\|_q^q&=\|x^{(\epsilon,q)}+x^{(1-\epsilon,q)}-v\|_q^q=\sum\limits_{i=1}^p |x_i^{(\epsilon,q)}+x_i^{(1-\epsilon,q)}-v_i|^q\\
&=\sum\limits_{i=1,i\in I^*}^p |x_i^{(\epsilon,q)}+x_i^{(1-\epsilon,q)}-v_i|^q+\sum\limits_{i=1,i\notin I^*}^p |x_i^{(\epsilon,q)}+x_i^{(1-\epsilon,q)}-v_i|^q\\
&=\sum\limits_{i=1,i\in I^*}^p \big||x_i^{(\epsilon,q)}|+|x_i^{(1-\epsilon,q)}|-\text{sgn}(x_i)v_i\big|^q+\sum\limits_{i=1,i\notin I^*}^p |x_i^{(1-\epsilon,q)}-v_i|^q\\
&=\sum\limits_{i=1,i\in I^*}^p \big(|x_i^{(\epsilon,q)}|+|x_i^{(1-\epsilon,q)}|-\text{sgn}(x_i)v_i\big)^q+\sum\limits_{i=1,i\notin I^*}^p |x_i^{(1-\epsilon,q)}-v_i|^q\\
&\geq \sum\limits_{i=1,i\in I^*}^p |x_i^{(\epsilon,q)}|^q+(|x_i^{(1-\epsilon,q)}|-\text{sgn}(x_i)v_i)^q+\sum\limits_{i=1,i\notin I^*}^p |x_i^{(1-\epsilon,q)}-v_i|^q\\
&=\|x^{(\epsilon,q)}\|_q^q+\|x^{(1-\epsilon,q)}-v\big\|_q^{q}> \epsilon^p\|x\|_{\epsilon q}^q
\end{align}
where we are using the fact that 
\begin{align}
&v\in V(\|x\|_{\epsilon q}), \|x^{(1-\epsilon,q)}\|_\infty=(1-\epsilon)\|x\|_{\epsilon,q}\implies |x_i^{(1-\epsilon,q)}|-\text{sgn}(x_i)v_i\geq 0, \ \forall i\in I_*.\\
&\because \text{With Minkowski inequality }(a^q+b^q)^{1/q}\leq (a^q)^{1/q}+(b^q)^{1/q}=a+b,\forall a\geq 0,b\geq 0, q\geq 1.\\
&\therefore \big(|x_i^{(\epsilon,q)}|+|x_i^{(1-\epsilon,q)}|-\text{sgn}(x_i)v_i\big)^q\geq |x_i^{(\epsilon,q)}|^q+ \big(|x_i^{(1-\epsilon,q)}|-\text{sgn}(x_i)v_i\big)^q \\
&\sum\limits_{i=1,i\in I^*}^p\big(|x_i^{(1-\epsilon,q)}|-\text{sgn}(x_i)v_i\big)^{q}+\sum\limits_{i=1,i\notin I^*}^p |x_i^{(1-\epsilon,q)}-v_i|^q\\
=&\sum\limits_{i=1,i\in I^*}^p\big|x_i^{(1-\epsilon,q)}-(x_i)v_i\big|^{q}+\sum\limits_{i=1,i\notin I^*}^p |x_i^{(1-\epsilon,q)}-v_i|^q=\|x^{(1-\epsilon,q)}-v\|_q^q.
\end{align}
Therefore, we conclude the proof of decomposition uniqueness.

Based on the $\epsilon$-decomposition, for any $\nu\geq 0$, we can see that $\forall x\in \{x\in\mathbb{R}^p:\|x\|_{\epsilon q}\leq \nu\}$, we have $x^{(1-\epsilon,q)}, x^{(\epsilon,q)}\in \mathbb{R}^p$ such that $x=x^{(1-\epsilon,q)}+x^{(\epsilon,q)},\|x^{(\epsilon,q)}\|_q= \epsilon\|x\|_{\epsilon q}\leq \epsilon \nu,\|x^{(1-\epsilon,q)}\|_\infty= (1-\epsilon)\|x\|_{\epsilon q}\leq (1-\epsilon) \nu$. Therefore, 
\begin{align}
\{x\in\mathbb{R}^p:\|x\|_{\epsilon q}\leq \nu\}\subseteq \{u+v:u,v\in\mathbb{R}^p,\|u\|_q\leq \epsilon \nu,\|v\|_\infty\leq (1-\epsilon) \nu\}. 
\end{align}
On the other hand, for any $u,v\in\mathbb{R}^p$ such that $\|u\|_q\leq \epsilon \nu,\|v\|_\infty\leq (1-\epsilon) \nu$. We prove that $x=u+v$ satisfying $\|x\|_{\epsilon q}\leq \nu$. Suppose it is false, i.e. $\|x\|_{\epsilon q}>\nu$. Then $\|v\|_\infty\leq (1-\epsilon)v< (1-\epsilon)\|x\|_{\epsilon q}$. From $\epsilon$-decomposition and similar argument from the proof above, we have
\begin{align}
(\epsilon \nu)^q \geq \|u\|_q^q&=\|x-v\|^q_q=\|x^{(\epsilon,q)}+x^{(1-\epsilon,q)}-v\|^q_q\\
&\geq\|x^{(\epsilon,q)}\|_q^q+\|x^{(1-\epsilon,q)}-v\big\|_q^{q}\geq \epsilon^q \|x\|_{\epsilon q}^q>\epsilon^q\nu^q. 
\end{align}
Contradict! Therefore we must have $\|x\|_{\epsilon q}\leq\nu$, implying 
\begin{align}
\{x\in\mathbb{R}^p:\|x\|_{\epsilon q}\leq \nu\}= \{u+v:u,v\in\mathbb{R}^p,\|u\|_q\leq \epsilon \nu,\|v\|_\infty\leq (1-\epsilon) \nu\}. 
\end{align}
For $q=0$, the results can be easily verified. For $0<q<1$, once again we consider proof by contradiction.  Suppose $\|x\|_{\epsilon q}>\nu$. Then $\|v\|_\infty\leq (1-\epsilon)v< (1-\epsilon)\|x\|_{\epsilon q}$.
Then 
\begin{align}
(1-\epsilon)\|x\|_{\epsilon q}>\|v\|_\infty&=\|x-u\|_\infty=\|x^{(1-\epsilon,q)}+x^{(\epsilon,q)}-u\|_\infty=\big\||x^{(1-\epsilon,q)}|+|x^{(\epsilon,q)}|-\text{sgn}(x)\odot u\big\|_\infty\\
&\geq \big\|x^{(1-\epsilon,q)}\|_\infty = (1-\epsilon)\|x\|_{\epsilon q}.
\end{align}
Contradict! We are using the fact that there must exist an $i$ such that $|x_i^{(\epsilon,q)}|-(\text{sgn}(x)\odot u)_i\geq 0$. If $|x_i^{(\epsilon,q)}|-(\text{sgn}(x)\odot u)_i<0,\forall i\geq 1$, then $\|u\|_{ q}>\|x^{(\epsilon,q)}\|_{ q}=\epsilon \|x\|_{\epsilon q}>\epsilon \nu$. Since $\|u\|_{q}\leq \epsilon \nu$, we get a contradiction.

\end{proof}
\begin{lemma*}
For any $\epsilon\in (0,1]$ and $q\in [1,\infty)$, the norm dual of the $\epsilon q$-norm is given by 
\begin{alignat}{2}
\|y\|_{\epsilon q}^*=\begin{cases}\epsilon\|y\|_{q/(q-1)}+(1-\epsilon)\|y\|_{1},\quad &q<\infty,\\
\|y\|_{1}, &q=\infty\end{cases}.
\end{alignat}
\end{lemma*}
\begin{proof}[Proof of Lemma \ref{lemma_dual_norm_primitive}]
The case $q=\infty$ is trivial. We only consider the case that $q<\infty$.

For $q<\infty$, 
\begin{align}
\|y\|_{\epsilon q}^*&=\sup\limits_{\|x\|_{\epsilon q}\leq 1}\langle y,x\rangle=\sup\limits_{\|u\|_{q}\leq \epsilon,\|v\|_{\infty}\leq 1-\epsilon}\langle y,u+v\rangle=\epsilon\sup\limits_{\|u\|_{q}\leq 1}\langle y,u\rangle+(1-\epsilon)\sup\limits_{\|v\|_{\infty}\leq 1}\langle y,x\rangle\\
&=\epsilon\|y\|_{q/(q-1)}+(1-\epsilon)\|y\|_1.
\end{align}
\end{proof}
\begin{lemma*}
First denote $\epsilon_g:=\frac{(1-\tau)w_g}{\tau+(1-\tau)w_g},\forall g\in\{1,...,G\}$ and $\alpha^*=(\alpha^*_1,...,\alpha^*_G)$ where $\alpha^*_g=\frac{\alpha_g}{\alpha_g-1}$. If $\alpha_g=\infty,\alpha_g^*=1$. The double sparsity norm $\|\cdot\|_\ds$ satisfies the following properties: $\forall \beta\in\mathbb{R}^p,x\in \mathbb{R}^p$,
\begin{align}
\|\beta\|_\ds =\sum\limits_{g\in \{1,...,G\}}(\tau+(1-\tau)w_g)\|\beta_{(g)}\|_{\epsilon_g \alpha_g^*}^*, \quad \textup{and}\quad  \|x\|_\ds^* = \max\limits_{g\in \{1,...,G\}}\frac{\|x_{(g)}\|_{\epsilon_g \alpha_g^*}}{\tau+(1-\tau)w_g}.
\end{align}
\end{lemma*}
\begin{proof}[Proof of Lemma \ref{lemma_dual_norm}]
\begin{align}
\|\beta\|_\ds&=\tau\|\beta\|_1+(1-\tau)\sum\limits_{g=1}^Gw_g\|\beta\|_{\alpha_g}=\sum\limits_{g=1}^G(\tau\|\beta_{(g)}\|_1+(1-\tau) w_g\|\beta_{(g)}\|_{\alpha_g})\\
&=\sum\limits_{g=1}^G[\frac{\tau}{\tau+(1-\tau)w_g}\|\beta_{(g)}\|_1+\frac{(1-\tau) w_g}{\tau+(1-\tau)w_g}\|\beta_{(g)}\|_{\alpha_g}]\big(\tau+(1-\tau)w_g\big)\\
&=\sum\limits_{g=1}^G\big[(1-\epsilon_g)\|\beta_{(g)}\|_1+\epsilon_g\|\beta_{(g)}\|_{\alpha_g}\big]\big(\tau+(1-\tau)w_g\big)=\sum\limits_{g=1}^G \big(\tau+(1-\tau)w_g\big)\|\beta_{(g)}\|_{\epsilon_g,\alpha_g^*}^*\\
\|x\|_\ds^* &= \sup\limits_{\|\beta\|_\ds\leq 1}\langle \beta,x \rangle= \sup\limits_{\beta}\inf\limits_{\nu\geq 0}\langle \beta,x \rangle-\nu(\|\beta\|_\ds-1)=\sup\limits_{\beta}\inf\limits_{\nu\geq 0}\langle \beta,x \rangle-\nu(\|\beta\|_\ds-1)\\
&=\inf\limits_{\nu\geq0}\sup\limits_{\beta}\sum\limits_{g=1}^G\beta_{(g)}x_{(g)}-\nu\tau\|\beta_{(g)}\|_1-\nu(1-\tau)w_g\|\beta_{(g)}\|_{\alpha_g}+\nu\\
&=\inf\limits_{\nu\geq0}\sum\limits_{g=1}^G\sup\limits_{\beta_{(g)}}\Big(\beta_{(g)}x_{(g)}-\nu\tau\|\beta_{(g)}\|_1-\nu(1-\tau)w_g\|\beta_{(g)}\|_{\alpha_g}\Big)+\nu\\
&=\inf\limits_{\nu\geq0}\sum\limits_{g=1}^G\inf_{x_1+x_2=x_{(g)}}\mathcal{I}(\|x_1\|_\infty\leq \nu\tau)+\mathcal{I}(\|x_2\|_{\alpha_g^*}\leq \nu(1-\tau)w_g)+\nu\\
&=\inf\limits_{\nu\geq0}\max\limits_{g\in\{1,...,G\}}\inf_{x_1+x_2=x_{(g)}}\mathcal{I}(\|x_1\|_\infty\leq \nu\tau)+\mathcal{I}(\|x_2\|_{\alpha_g^*}\leq \nu(1-\tau)w_g)+\nu\\
&=\inf\limits_{\nu\geq0}\max\limits_{g\in\{1,...,G\}}\sup\limits_{\beta_{(g)}}\Big(\beta_{(g)}x_{(g)}-\nu\tau\|\beta_{(g)}\|_1-\nu(1-\tau)w_g\|\beta_{(g)}\|_{\alpha_g}\Big)+\nu\\
&=\max\limits_{g\in\{1,...,G\}}\sup\limits_{\|\beta_{(g)}\|_\ds\leq 1}\Big(\beta_{(g)}x_{(g)}\Big)=\max\limits_{g\in\{1,...,G\}}\sup\limits_{(1-\epsilon_g)\|\beta'_{(g)}\|_1+\epsilon_g\|\beta'_{(g)}\|_{\alpha_g}\leq 1 }\Big(\frac{\beta'_{(g)}x_{(g)}}{\tau+(1-\tau)w_g}\Big)\\
&=\max\limits_{g\in\{1,...,G\}}\frac{\|x_{(g)}\|_{\epsilon_g \alpha_g^*}}{\tau+(1-\tau)w_g}
\end{align}

The above derivations are using the following facts:
\begin{align}
\Omega_1(\beta) :&= \alpha_1\|\beta\|_1,\Omega_2(\beta) :=\alpha_2\|\beta\|_q\\
\Omega_1^*(x) &= \sup\limits_{\beta} \ (x^\top\beta-\alpha_1\|\beta\|_1)=\alpha_1\mathcal{I}(\|x\|_\infty\leq \alpha_1)=\mathcal{I}(\|x\|_\infty\leq \alpha_1)\\
\Omega_2^*(x) &= \sup\limits_{\beta} \ (x^\top\beta-\alpha_2\|\beta\|_q)=\alpha_2\mathcal{I}(\|x\|_{q/(q-1)}\leq \alpha_2)=\mathcal{I}(\|x\|_{q/(q-1)}\leq \alpha_2)\\
(\Omega_1+\Omega_2)^*(x)&=\inf\limits_{x_1+x_2=x} \ \big[\mathcal{I}(\|x_1\|_\infty\leq \alpha_1)+\mathcal{I}(\|x_2\|_{q/(q-1)}\leq \alpha_2)\big] \quad (\text{From Theorem 16.4 of \cite{rockafellarConvexAnalysis1970}}).
\end{align}
\end{proof}
\begin{lemma*}
The dual formulation of optimization problem \ref{DS-LASSO} is given by a convex optimization below
\begin{equation}
\begin{aligned}
& \underset{\theta,u_g,v_g}{\text{max}}
& &  \|y\|_2^2-\|\frac{\lambda \theta}{2}-y\|_2^2\\
& & & u_g+v_g = X_{(g)}^\top \theta,\forall g=1,...,G,\\
& & &\|u_g\|_{\alpha_g^*} \leq\epsilon_g(\tau+(1-\tau)w_g),\forall g=1,...,G,\\
& & & \|v_g\|_{\infty} \leq(1-\epsilon_g)\big(\tau+(1-\tau)w_g\big),\forall g=1,...,G.
\end{aligned}
\end{equation}

\end{lemma*}
\begin{proof}[Proof of Lemma \ref{lemma_dual_dslasso}]
We derive the dual problem via the Lagrangian multiplier method. 

Denote $z=y-X\beta$ and $\lambda\theta$ as the Lagrangian multiplier, the dual function is given by 
\begin{align}
D(\theta)&=\inf_{z,\beta}\|z\|_2^2+\lambda \|\beta\|_\ds+\langle\lambda\theta.y-X\beta-z\rangle\\
&=\inf_{z,\beta}\lambda\sum\limits_{g=1}^G\big[\|\beta_{(g)}\|_\ds-\langle\theta, X_{(g)}\beta_{(g)}\rangle\big]-\lambda\langle\theta,z\rangle+\|z\|_2^2+\lambda\langle\mu,y\rangle\\
&=-\lambda\sum\limits_{g=1}^G\sup_{\beta_{(g)}}\big[\langle X_{(g)}^\top\theta, \beta_{(g)}\rangle-\|\beta_{(g)}\|_\ds\big]-\frac{\lambda^2\|\theta\|_2^2}{4}+\lambda\langle\theta,y\rangle\\
&=-\lambda\sum\limits_{g=1}^G\sup_{\beta_{(g)}}\big[\langle X_{(g)}^\top\theta, \beta_{(g)}\rangle-\big(\tau+(1-\tau)w_g\big)\|\beta_{(g)}\|_{\epsilon_g \alpha_g^*}^*\big]+ \|y\|^2_2-\big\|y-\frac{\lambda \theta}{2}\big\|_2^2\\
&=-\lambda\sum\limits_{g=1}^G \mathcal{I}\big(\|X_{(g)}^\top\theta\|_{\epsilon_g \alpha_g^*}\leq (\tau+(1-\tau)w_g\big)+\|y\|^2_2-\big\|y-\frac{\lambda \theta}{2}\big\|_2^2
\end{align} 
where we are using the fact that 
\begin{align}
&\inf_z -\lambda\langle \theta,z\rangle+\|z\|_2^2=\inf_z\frac{\lambda^2\|\theta\|^2}{4}+\big\|z-\frac{\lambda \theta}{2}\big\|_2^2=\frac{\lambda^2\|\theta\|^2}{4},\\
&\Omega(\beta) = \alpha\|\beta\|_{\epsilon q}^*,\Omega^*(x)=\sup_\beta(\langle x,\beta\rangle-\alpha\|\beta\|_{\epsilon q}^*)=\alpha\mathcal{I}\big(\|x\|_{\epsilon q}\leq \alpha\big).
\end{align}
Therefore the dual problem is 
\begin{equation}
\begin{aligned}
& \underset{\theta}{\text{max}}
& &  \|y\|_2^2-\|\frac{\lambda \theta}{2}-y\|_2^2\\
& \text{ subject to }& &   \|X_{(g)}^\top\theta\|_{\epsilon_g \alpha_g^*}\leq \tau+(1-\tau)w_g,\forall g=1,...,G
\end{aligned}
\end{equation}
The set representation (\ref{setrep}) gives the program (\ref{dual_dslasso}). Finally $\|\frac{\lambda\theta}{2}-y\|_2^2$ is a convex function in $\theta$. $\|u_g\|_{\alpha_g^*}$ and $\|v_g\|_\infty$ are convex function. $u_g+v_g=X^{\top}_{(g)}\theta$ are affine constraints. Therefore it is a convex optimization. One can solve it using algorithm for convex programming and guarantee an optimal value. It concludes the proof.
\end{proof}
\begin{theorem*}
For any Gaussian random design matrix $X\in\mathbb{R}^{n\times p}$ with i.i.d. $\mathcal{N}(0,\Sigma)$ rows, the following inequalities hold for all $v\in \mathbb{R^p}$ with probability $1-c\exp(-c'n)$, in which $c,c'$ are some constants,
\begin{align}
\|Xv\|_2&\geq \frac{\sqrt{n}}{4}\|\Sigma^{1/2}v\|_{2}-3\kappa_1\|v\|_\ds,\\
\|Xv\|_2&\geq \Big(\frac{\sqrt{n}}{4}\lambda_{\min}(\Sigma^{1/2})-3\kappa_1\kappa_2\Big)\|v\|_2,\quad \textup{ for } n> \frac{144\kappa_1^2\kappa_2^2}{\lambda_{\min}(\Sigma)},
\end{align}
where 
\begin{align*}
\kappa_1:&=\max\limits_{g\in\{1,...,G\}} \Bigg[\frac{p_{g}^{(1/2-1/\alpha_g)_+}s\sqrt{\Tr{(\Sigma_{(g)(g)})}}}{\tau+(1-\tau)w_g}\bigwedge \frac{3\sqrt{\log{p_g}\cdot\max_{i\in (g)}\Sigma_{ii}}}{(\tau+(1-\tau)w_g)(1-\epsilon_g+\epsilon_g p_g^{1/\alpha_g-1})}\Bigg]\\
\kappa_2:&= \sqrt{G}\max\limits_{g\in\{1,...,G\}}(\tau p_g^{1/2}+(1-\tau)w_g p_g^{(1/\alpha_g-1/2)_+}),\ \epsilon_g=\frac{(1-\tau)w_g}{\tau+(1-\tau)w_g},\forall g\in\{1,...,G\}.
\end{align*}

The second inequality \ref{rsc} is called restricted eigenvalues condition\cite{raskuttiRestrictedEigenvalueProperties}.
\end{theorem*}
\begin{proof}[Proof of Theorem \ref{thm_rec}]
We first consider the set $V(r)$ indexed by $r$, i.e. $V(r):=\{v\in\mathbb{R}^p|\|\Sigma^{1/2}v\|_2=1,\|v\|_{\ds}\leq r\}$. Define a random variable 
\begin{align}
M(r,X):=1-\inf\limits_{v\in V(r)}\frac{\|Xv\|_2}{\sqrt{n}}=\sup\limits_{v\in V(r)}\big\{1-\frac{\|Xv\|_2}{\sqrt{n}}\big\}.
\end{align}

We first consider an upper bound for $\mathbb{E}[M(r,X)]$. Let $S^{n-1}=\{u\in \mathbb{R}^n|\|u\|_2=1\}$. Then we have 
\begin{align}
\mathbb{E}[M(r,X)]&=1-\frac{1}{\sqrt{n}}\mathbb{E}[\inf\limits_{v\in V(r)}\|Xv\|_2]=1-\frac{1}{\sqrt{n}}\mathbb{E}[\inf\limits_{v\in V(r)}\sup\limits_{u\in S^{n-1}}u^\top Xv]\\
&=1+\frac{1}{\sqrt{n}}\mathbb{E}[\sup\limits_{v\in V(r)}\inf\limits_{u\in S^{n-1}}u^\top Xv]
\end{align}
by noticing that $-\inf f(x)=\sup [-f(x)], \inf\limits_{u\in S^{n-1}} [-u^\top Xv]=\inf\limits_{u\in S^{n-1}} [u^\top Xv]$.

Note that $X$ can be written as $W\Sigma^{1/2}$ where $W$ is Gaussian matrix in $\mathbb{R}^{n\times p}$ with i.i.d $\mathcal{N}(0,1)$ entries. We define a zero-mean Gaussian random variable $Y_{u,v}:=u^TXv$ for each pair $(u,v)\in S^{n-1}\times V(r)$ and write $\tilde{v}=\Sigma^{1/2}v$.
\begin{align}
Y_{u,v}&=u^\top W\Sigma^{1/2}v=u^\top W\tilde{v}\\
\sigma^2(Y_{u,v}-Y_{u',v'})&=\mathbb{E}\big[(u^\top W\tilde{v}-u'^\top W\tilde{v}')^2\big]=\mathbb{E}\Big[\big(\sum\limits_{1\leq i\leq n,1\leq j\leq p}W_{i,j}(u_i\tilde{v}_j-u'_i\tilde{v}'_j)\big)^2\Big]\\
&=\sum\limits_{1\leq i\leq n,1\leq j\leq p}(u_i\tilde{v}_j-u'_i\tilde{v}'_j)^2=\sum\limits_{1\leq i\leq n,1\leq j\leq p}((u_i-u'_i)\tilde{v}_j-u'_i(\tilde{v}_j-\tilde{v}'_j))^2\\
&=\|u-u'\|_2^2\|\tilde{v}\|^2_2+\|u'\|_2^2\|\tilde{v}-\tilde{v}'\|_2^2-2(u^\top u'-\|u'\|_2)(\|\tilde{v}\|_2^2-\tilde{v}'^\top \tilde{v})\\
&\leq \|u-u'\|_2^2+\|\tilde{v}-\tilde{v}'\|_2^2. 
\end{align}
The derivation follows immediately with the following facts: 
\begin{align*}
\mathbb{E}{W_{ij}}&=0,\mathbb{E}{W^2_{ij}}=1,\mathbb{E}{W_{ij}W_{i'j'}}=0, \textup{ if } i\neq i' \text{ or } j\neq j',\\
\|u\|_2&=\|u'\|=1,\|\tilde{v}\|_2=\|\Sigma^{1/2}v\|=\|\Sigma^{1/2}v'\|=\|\tilde{v}'\|_2=1\\
u'^\top u&\leq \|u'\|_2\|u\|_2,\tilde{v}'^\top \tilde{v}\leq \|\tilde{v}'\|_2\|\tilde{v}\|_2
\end{align*}

Moreover, if $v=v'$, we then have the equality hold, i.e. $\sigma^2(Y_{u,v}-Y_{u',v'})=\|u-u'\|_2^2$. 

We define another zero-mean Gaussian process $Z_{u,v}$ indexed by $S^{n-1}\times V(r)$ given by the following expression:
\begin{align}
Z_{u,v}=s^\top u+ t^\top \Sigma^{1/2}v
\end{align}
where $s\sim \mathcal{N}(0,I_{n\times n}),t\sim\mathcal{N}(0,I_{p\times p})$. It is easy to verify that 
\begin{align}
\sigma^2(Z_{u,v}-Z_{u',v'}) = \mathbb{E}\big[(s^\top (u-u')+ t^\top \Sigma^{1/2}(v-v'))^2\big]=\|u-u'\|_2^2+\|\tilde{v}-\tilde{v}'\|_2^2.
\end{align}

Then according to Sudakov-Fernique inequality, we have 
\begin{align}
\mathbb{E}[M(r,X)]&=1+\frac{1}{\sqrt{n}}\mathbb{E}[\sup\limits_{v\in V(r)}\inf\limits_{u\in S^{n-1}}u^\top Xv]\\
&=1+\frac{1}{\sqrt{n}}\mathbb{E}[\sup\limits_{v\in V(r)}\inf\limits_{u\in S^{n-1}}Y_{u,v}]\leq 1+\frac{1}{\sqrt{n}}\mathbb{E}[\sup\limits_{v\in V(r)}\inf\limits_{u\in S^{n-1}}Z_{u,v}]\\
&= 1-\frac{1}{\sqrt{n}}\mathbb{E}\|s\|_2+\frac{1}{\sqrt{n}}\mathbb{E}[\sup\limits_{v\in V(r)}t^\top \Sigma^{1/2}v]\leq \frac{1}{4}+\frac{1}{\sqrt{n}}\mathbb{E}[\sup\limits_{v\in V(r)}t^\top \Sigma^{1/2}v]\\
&\leq \frac{1}{4}+\frac{1}{\sqrt{n}}\mathbb{E}[\sup\limits_{v\in V(r)}\|t^\top \Sigma^{1/2}\|_{\ds}^*\|v\|_{\ds}]\leq \frac{1}{4}+\frac{r}{\sqrt{n}}\mathbb{E}[\|\Sigma^{1/2}t\|_{\ds}^*]\\
&\leq \frac{1}{4}+\frac{r}{\sqrt{n}}\max\limits_{g\in\{1,...,G\}} \Big[\frac{p_{g}^{(1/2-1/\alpha_g)_+}\sqrt{\Tr{(\Sigma_{(g)(g)})}}}{\tau+(1-\tau)w_g}\wedge \frac{3\sqrt{\log{p_g}\cdot\max_{i\in (g)}\Sigma_{ii}}}{(\tau+(1-\tau)w_g)(1-\epsilon_g+\epsilon_g p_g^{1/\alpha_g-1})}\Big]
\end{align}
by noticing that 
\begin{itemize}
  \item 
$\mathbb{E}[\|s\|_2]\geq \frac{3}{4}\sqrt{n}$ for all $n\geq 10$ from standard $\chi$ distribution tail bound result;  
  \item Generalized Cauchy–Schwarz inequality: $t^\top \Sigma^{1/2}v\leq \|t^\top \Sigma^{1/2}\|_{\ds}^*\|v\|_{\ds}$;
  \item
$\|\Sigma^{1/2}t\|_\ds^* = \max\limits_{g\in \{1,...,G\}}\frac{\|\Sigma_{(g),:}^{1/2}t\|_{\epsilon_g \alpha_g^*}}{\tau+(1-\tau)w_g}$, \\
$\|\Sigma_{(g),:}^{1/2}t\|_{\epsilon_g \alpha_g^*}\leq  \big[p_{g}^{(1/\alpha_g^*-1/2)_+}\|\Sigma_{(g),:}^{1/2}t\|_{2}\big] \bigwedge \frac{p_g^{1/\alpha_g^*}\|\Sigma_{(g),:}^{1/2}t\|_\infty}{p_g^{1/\alpha_g^*}(1-\epsilon_g)+\epsilon_g}$
\item $\mathbb{E}[\|\Sigma_{(g),:}^{1/2}t\|_2]\leq \sqrt{\mathbb{E}[\|\Sigma_{(g),:}^{1/2}t\|_2^2]}=\sqrt{\Tr(\Sigma_{(g)(g)})},\\
\mathbb{E}[\|\Sigma_{(g),:}^{1/2}t\|_\infty]\leq 3\sqrt{\log{p_g}\cdot\max\limits_{i\in (g)}\Sigma_{ii}}$ (see Equation (3.13) of \cite{ledouxProbabilityBanachSpaces1991} for the upper bound on $\mathbb{E}[\|\Sigma^{1/2}t\|_\infty]$). 
\end{itemize}

We denote $t(r):=\frac{1}{4}+\frac{r}{\sqrt{n}}\max\limits_{g\in\{1,...,G\}} \Big[\frac{p_{g}^{(1/2-1/\alpha_g)_+}\sqrt{\Tr{(\Sigma_{(g)(g)})}}}{\tau+(1-\tau)w_g}\wedge \frac{3\sqrt{\log{p_g}\cdot\max_{i\in (g)}\Sigma_{ii}}}{(\tau+(1-\tau)w_g)(1-\epsilon_g+\epsilon_g p_g^{1/\alpha_g-1})}\Big]$ and rewrite $M(r,X)$ as $H(r,W):=1+\frac{1}{\sqrt{n}}\sup_{v\in V(r)}(-\|W\Sigma^{1/2}v\|_2)$ and show that $H(r,W)$ is a Lipschitz function, i.e. $H(r,W')-H(r,W)\leq \frac{1}{\sqrt{n}}\|W-W'\|_F$.

Since $\|Xv\|_2$ is continuous for $v$ and $V(r)$ is closed and bounded, we denote $\hat{v}$ as one maximizer of $\sup_{v\in V(r)}(-\|Xv\|_2)$. 
\begin{align}
\sqrt{n}[H(r,X)-H(r,X')]=&\sup_{v\in V(r)}(-\|W\Sigma^{1/2}v\|_2)-\sup_{v\in V(r)}(-\|W'\Sigma^{1/2}v\|_2)\leq -\|W\Sigma^{1/2}\hat{v}\|_2+\|W'\Sigma^{1/2}\hat{v}\|_2\\&\leq \sup_{v\in V(r)} [\|W'\Sigma^{1/2}v\|_2-\|W\Sigma^{1/2}v\|_2]\leq \sup_{v\in V(r)} \|W'\Sigma^{1/2}v-W\Sigma^{1/2}v\|_2\\
&\leq \sup_{v\in V(r)} \|W-W'\|_2\|\Sigma^{1/2}v\|_2\leq \|W-W'\|_2\leq \|W-W'\|_F
\end{align}

Therefore, from concentration bound for Lipschitz functions of Gaussian random variables (see for example Theorem 3.8 from \cite{massartConcentrationInequalitiesModel2007}), we have the following:
\begin{align}
P\Big(|M(r,X)-\mathbb{E}[M(r,X)]|\geq t(r)/2\Big)\leq 2\exp\big(-nt(r)^2/8\big),
\end{align}
implying 
\begin{align}
P\Big(M(r,X)\geq 3t(r)/2\Big)\leq 2\exp\big(-nt(r)^2/8\big).
\end{align}

If we use the notation in the peeling argument (Lemma \ref{peeling}). Define an event $\mathcal{E}:=\{\exists v\in\mathbb{R}^p,\ \textup{s.t.}\  \|\Sigma^{1/2}v\|_2=1, (1-\|Xv\|_2/\sqrt{n})\geq 3t(\|v\|_{\ds})\}$.  Let $f(v,X)=1-\frac{1}{\sqrt{n}}\|Xv\|_2,h(v)=\|v\|_{\ds},g(r)=3t(r)/2$, $a_n=n, A=\{v\in \mathbb{R}^p|\|\Sigma^{1/2}v\|_2=1\}$. We then have $g(r)\geq 3/8$, $g(r)$ non-negative and strictly increasing and $h(v)$ non-negative and increasing. From the lower bound of $g(r)$ we can set $\mu = \frac{3}{8}$. Moreover, 
\begin{align}
P\Big(\sup_{v\in A,h(v)\leq r} f(v,X)\geq g(r)\Big)&=P(\sup_{\|\Sigma^{1/2}v\|_2=1,\|v\|_{\ds}\leq r} 1-\frac{1}{\sqrt{n}}\|Xv\|_2\geq 3t(r)/2)\\
&\leq 2\exp\big(-nt(r)^2/8\big) = 2\exp\big(-\frac{1}{18}a_n g^2(r)\big)
\end{align}
According to peeling argument, we have 
\begin{align}
P(\mathcal{E})\leq 2\frac{\exp{(-\frac{4n}{18}\frac{9}{64})}}{1-\exp{(-\frac{4n}{18}\frac{9}{64})}}=\frac{2\exp(-n/32)}{1-2\exp(-n/32)} \implies P(\mathcal{E}^c)\geq 1-\frac{2\exp(-n/32)}{1-2\exp(-n/32)}
\end{align}
Therefore, with probability at least $1-\frac{2\exp(-n/32)}{1-2\exp(-n/32)}$, for all $v\in \mathbb{R}^p$ with $\|\Sigma^{1/2}v\|_2=1$, we have 
\begin{align}
\frac{\|Xv\|_2}{\sqrt{n}}&\geq \frac{1}{4}-\frac{3\|v\|_\ds}{\sqrt{n}}\max\limits_{g\in\{1,...,G\}} \Big[\frac{p_{g}^{(1/2-1/\alpha_g)_+}\sqrt{\Tr{(\Sigma_{(g)(g)})}}}{\tau+(1-\tau)w_g}\wedge \frac{3\sqrt{\log{p_g}\cdot\max_{i\in (g)}\Sigma_{ii}}}{(\tau+(1-\tau)w_g)(1-\epsilon_g+\epsilon_g p_g^{1/\alpha_g-1})}\Big]
\end{align}
implying for all $v\in \mathbb{R}^p$,
\begin{align}
&\|Xv\|_2\\
\geq &\frac{\sqrt{n}}{4}\|\Sigma^{1/2}v\|_{2}-3\|v\|_\ds\max\limits_{g\in\{1,...,G\}} \Big[\frac{p_{g}^{(1/2-1/\alpha_g)_+}\sqrt{\Tr{(\Sigma_{(g)(g)})}}}{\tau+(1-\tau)w_g}\wedge \frac{3\sqrt{\log{p_g}\cdot\max_{i\in (g)}\Sigma_{ii}}}{(\tau+(1-\tau)w_g)(1-\epsilon_g+\epsilon_g p_g^{1/\alpha_g-1})}\Big]  
\end{align}
The restricted eigenvalues condition holds with the fact that $\|\Sigma^{1/2}v\|_2\geq \lambda_{\min}(\Sigma^{1/2})\|v\|_2$ and 
\begin{align}
\|v\|_\ds&=\sum\limits_{g=1}^G \tau\|v_{(g)}\|_1+(1-\tau)w_g\|v_{(g)}\|_{\alpha_g}\leq \sum\limits_{g=1}^G \tau p_g^{1/2}\|v_{(g)}\|_2+(1-\tau)w_g p_g^{(1/\alpha_g-1/2)_+}\|v_{(g)}\|_2\\
&\leq \max\limits_{g\in\{1,...,G\}}(\tau p_g^{1/2}+(1-\tau)w_g p_g^{(1/\alpha_g-1/2)_+})\sum\limits_{g=1}^G\|v_{(g)}\|_2\\
&\leq \max\limits_{g\in\{1,...,G\}}(\tau p_g^{1/2}+(1-\tau)w_g p_g^{(1/\alpha_g-1/2)_+})\sqrt{G} \|v\|_2
\end{align}
\end{proof}

\begin{theorem*}
Suppose that $X$ satisfies the block column normalization condition \ref{column_normalization} and the observation noise $\varepsilon$ is sub-Gaussian, satisfying sub-Gaussian tail \ref{sub_gaussian}. Then we have with probability at least $1-\frac{2}{G^2}$, the following inequality holds
\begin{align}
\frac{\|X^\top \varepsilon\|_\ds^*}{n}\leq \max\limits_{g\in\{1,...,G\}}\sigma\frac{\bigg[\frac{\epsilon_gp_g^{(1/\alpha_g-1/2)_+}+(1-\epsilon_g)p_g^{1/2}}{\sqrt{n}}\Big(\big[p_{g}^{(1/2-1/\alpha_g)_+}\sqrt{p_g}\big] \bigwedge \frac{{p_g}^{1/\alpha_g^*}\sqrt{2\log{p_g}}}{{p_g}^{1/\alpha_g^*}(1-\epsilon_g)+\epsilon_g} \Big)+  \sqrt{\frac{6\log{G}}{n}}\bigg]}{\tau+(1-\tau)w_g}.
\end{align}
\end{theorem*}
\begin{proof}[Proof of Theorem \ref{thm_lambda}] 
Without loss of generality, we assume $\sigma^2=1$. Since $\|X^\top \varepsilon\|_\ds^* = \max\limits_{g\in \{1,...,G\}}\frac{\|X_{(g)}^\top \varepsilon\|_{\epsilon_g \alpha_g^*}}{\tau+(1-\tau)w_g},$ we consider a probability bound for $\|X_{(g)}^\top \varepsilon\|_{\epsilon_g \alpha_g^*}$ and the bound for $\|X^\top \varepsilon\|_\ds^*$ can be easily derived by union bound. 

Notice that for any pair $\varepsilon,\varepsilon'\in \mathbb{R}^n$, we have 
\begin{align}
\frac{1}{n}\big|\|X_{(g)}^\top \varepsilon\|_{\epsilon_g \alpha_g^*}-\|X_{(g)}^\top \varepsilon'\|_{\epsilon_g \alpha_g^*}\big|&\leq \frac{1}{n}\|X_{(g)}^\top(\varepsilon-\varepsilon')\|_{\epsilon_g \alpha_g^*}=\frac{1}{n}\sup\limits_{u:\|u\|_{\epsilon_g \alpha_g^*}^*\leq 1}\langle u,X_{(g)}^\top(\varepsilon-\varepsilon') \rangle\\
&=\frac{1}{n}\sup\limits_{u:\|u\|_{\epsilon_g \alpha_g^*}^*\leq 1}\langle X_{(g)}u,(\varepsilon-\varepsilon') \rangle\leq \frac{1}{n}\|X\|_{(g)\to 2}\|\varepsilon-\varepsilon'\|_2^2\\
&=\frac{1}{\sqrt{n}}\|\varepsilon-\varepsilon'\|_2, 
\end{align}
where we are using the column normalization assumption that
\begin{align}
\sup\limits_{u:\|u\|_{\epsilon_g \alpha_g^*}^*\leq 1}\|X_{(g)}u\|_2=\sup\limits_{u:\epsilon_g\|u\|_{\alpha_g}+(1-\epsilon_g)\|u\|_{1}\leq 1}\|X_{(g)}u\|_2=\|X\|_{(g)\to 2}=\sqrt{n}
\end{align}

Therefore, by Gaussian concentration of measure for Lipschitz function (see for example Theorem 3.8 from \cite{massartConcentrationInequalitiesModel2007}), we have 
\begin{align}
P(\frac{1}{n}\|X_{(g)}^\top \varepsilon\|_{\epsilon_g \alpha_g^*}\geq \frac{1}{n}\mathbb{E}[\|X_{(g)}^\top \varepsilon\|_{\epsilon_g \alpha_g^*}]+ \delta) &\leq P(\frac{1}{n}\big|\|X_{(g)}^\top \varepsilon\|_{\epsilon_g \alpha_g^*}-\mathbb{E}[\|X_{(g)}^\top \varepsilon\|_{\epsilon_g \alpha_g^*}]\big|\geq \delta)\\
&\leq 2\exp(-n\delta^2/2), \ \forall \delta>0
\end{align}

Denote $Y_u = \frac{1}{n}\langle u,X_{(g)}^\top \varepsilon\rangle$ and $Z_u=\frac{\epsilon_gp_g^{(1/\alpha_g-1/2)_+}+(1-\epsilon_g)p_g^{1/2}}{\sqrt{n}}\langle u,\nu\rangle$ where $u\in \mathbb{R}^{p_g}$ and $\nu\sim\mathcal{N}(0,I_{p_g})$. Then we have $\frac{1}{n}\|X_{(g)}^\top \varepsilon\|_{\epsilon_g \alpha_g^*}=\sup\limits_{u:\|u\|_{\epsilon_g \alpha_g^*}^*\leq 1}Y_u$. It is easy to verify that for any pair $u,u'$ in $\{{u\in\mathbb{R}^{p_g}:\|u\|_{\epsilon_g \alpha_g^*}^*\leq 1}\}$,
\begin{align}
\mathbb{E}[Y_u]&=\mathbb{E}[Z_u]=0\\
\mathbb{E}[Y_u-Y_{u'}]&=\mathbb{E}[Z_u-Z_u']=0\\
\text{Var}(Y_u-Y_{u'})&= \mathbb{E}[(Y_u-Y_{u'})^2]=\frac{1}{n^2}\mathbb{E}[(\langle u-u',X_{(g)}^\top \varepsilon\rangle)^2]=\frac{1}{n^2}\|X_{(g)}(u-u')\|_2^2\\
&\leq \frac{1}{n^2}\Big\|X_{(g)}\frac{(u-u')}{\|u-u'\|_{\epsilon_g \alpha_g^*}^*}\Big\|^2_2\big[\|u-u'\|_{\epsilon_g \alpha_g^*}^*\big]^2\\
&\leq \frac{1}{n^2}\|X\|^2_{(g)\to2}\big[\epsilon_gp_g^{(1/\alpha_g-1/2)_+}+(1-\epsilon_g)p_g^{1/2}\big]^2\|u-u'\|_2^2\\
&=\frac{\big[\epsilon_gp_g^{(1/\alpha_g-1/2)_+}+(1-\epsilon_g)p_g^{1/2}\big]^2}{n}\|u-u'\|_2^2\\
\text{Var}(Z_u-Z_{u'})&= \mathbb{E}[(Z_u-Z_{u'})^2]=\frac{\big[\epsilon_gp_g^{(1/\alpha_g-1/2)_+}+(1-\epsilon_g)p_g^{1/2}\big]^2}{n}\mathbb{E}[(\langle u-u',\nu\rangle)^2]\\
&=\frac{\big[\epsilon_gp_g^{(1/\alpha_g-1/2)_+}+(1-\epsilon_g)p_g^{1/2}\big]^2}{n}\|u-u'\|_2^2.
\end{align}

Denote $\kappa^{(g)}:=\big[\epsilon_gp_g^{(1/\alpha_g-1/2)_+}+(1-\epsilon_g)p_g^{1/2}\big]^2$.

Therefore, with Sudakov-Fernique inequality, we have
\begin{align}
\frac{1}{n}\mathbb{E}[\|X_{(g)}^\top \varepsilon\|_{\epsilon_g \alpha_g^*}]&=\mathbb{E}[\sup\limits_{u:\|u\|_{\epsilon_g \alpha_g^*}^*\leq 1}Y_u]\leq \mathbb{E}[\sup\limits_{u:\|u\|_{\epsilon_g \alpha_g^*}^*\leq 1}Z_u]=\mathbb{E}[\sup\limits_{u:\|u\|_{\epsilon_g \alpha_g^*}^*\leq 1}\sqrt{\frac{\kappa^{(g)}}{n}}\langle u,\nu\rangle]\\
&\leq \sqrt{\frac{\kappa^{(g)}}{n}}\mathbb{E}[\sup\limits_{u:\|u\|_{\epsilon_g \alpha_g^*}^*\leq 1}\|u\|_{\epsilon_g \alpha_g^*}^*\|\nu\|_{\epsilon_g \alpha_g^*}\Big]\leq \sqrt{\frac{\kappa^{(g)}}{n}}\mathbb{E}[\|\nu\|_{\epsilon_g \alpha_g^*}]\\
&\leq \sqrt{\frac{\kappa^{(g)}}{n}}\Big(\big[p_{g}^{(1/2-1/\alpha_g)_+}\sqrt{p_g}\big] \bigwedge \frac{{p_g}^{1/\alpha_g^*}\sqrt{2\log{p_g}}}{{p_g}^{1/\alpha_g^*}(1-\epsilon_g)+\epsilon_g} \Big)
\end{align}
in which we are using the following facts 

\begin{align}
\mathbb{E}[\|\nu\|_{\epsilon_g \alpha_g^*}]&\leq p_{g}^{(1/2-1/\alpha_g)_+}\mathbb{E}[\|\nu\|_2]\leq p_{g}^{(1/2-1/\alpha_g)_+}\sqrt{\mathbb{E}[\|\nu\|_2^2]}=p_{g}^{(1/2-1/\alpha_g)_+}\sqrt{p_g},\\
\mathbb{E}[\|\nu\|_{\epsilon_g \alpha_g^*}]&\leq \frac{{p_g}^{1/\alpha_g^*}\mathbb{E}[\|\nu\|_\infty]}{{p_g}^{1/\alpha_g^*}(1-\epsilon_g)+\epsilon_g}\\
\mathbb{E}[\|\nu\|_\infty]&\leq \sqrt{2\log{p_g}}.
\end{align}

Then we have 
\begin{align}
P\bigg(\frac{1}{n}\|X_{(g)}^\top \varepsilon\|_{\epsilon_g \alpha_g^*}\geq \sqrt{\frac{\kappa^{(g)}}{n}}\Big(\big[p_{g}^{(1/2-1/\alpha_g)_+}\sqrt{p_g}\big] \bigwedge \frac{{p_g}^{1/\alpha_g^*}\sqrt{2\log{p_g}}}{{p_g}^{1/\alpha_g^*}(1-\epsilon_g)+\epsilon_g} \Big)+ \delta\bigg) \leq2\exp(-n\delta^2/2), \ \forall \delta>0
\end{align}
If we set $\delta = \sqrt{\frac{6\log{G}}{n}}$ and with union bound, we have 
\begin{align}
&P\bigg(\exists g\in\{1,...,G\}, \frac{1}{n}\frac{\|X_{(g)}^\top \varepsilon\|_{\epsilon_g \alpha_g^*}}{\tau+(1-\tau)w_g}\geq \frac{1}{\tau+(1-\tau)w_g}\times\notag\\
&\quad\quad\quad\Big[\sqrt{\frac{\kappa^{(g)}}{n}}\Big(\big[p_{g}^{(1/2-1/\alpha_g)_+}\sqrt{p_g}\big] \bigwedge \frac{{p_g}^{1/\alpha_g^*}\sqrt{2\log{p_g}}}{{p_g}^{1/\alpha_g^*}(1-\epsilon_g)+\epsilon_g} \Big)+  \sqrt{\frac{6\log{G}}{n}}\Big]\bigg) \leq \frac{2}{G^2}
\end{align}

Denote $\Psi = \max\limits_{g\in\{1,...,G\}}\frac{1}{\tau+(1-\tau)w_g}\bigg[\sqrt{\frac{\kappa^{(g)}}{n}}\Big(\big[p_{g}^{(1/2-1/\alpha_g)_+}\sqrt{p_g}\big] \bigwedge \frac{{p_g}^{1/\alpha_g^*}\sqrt{2\log{p_g}}}{{p_g}^{1/\alpha_g^*}(1-\epsilon_g)+\epsilon_g} \Big)+  \sqrt{\frac{6\log{G}}{n}}\bigg]$, then we have 
\begin{align}
&P\bigg(\exists g\in\{1,...,G\}, \frac{1}{n}\frac{\|X_{(g)}^\top \varepsilon\|_{\epsilon_g \alpha_g^*}}{\tau+(1-\tau)w_g}\geq \Psi\bigg) \leq \frac{2}{G^2}\\
\implies &P\bigg(\frac{\|X^\top \varepsilon\|_\ds^*}{n}\geq \Psi\bigg) \leq \frac{2}{G^2}
\end{align}
\end{proof}

\begin{theorem*}
Suppose that the design matrix satisfies the column normalization condition \ref{column_normalization} and the restricted eigenvalues condition (\ref{rsc}). Moreover, the noise $\varepsilon$ is sub-Gaussian \ref{sub_gaussian}. The generalized double sparsity estimator $\hat{\beta}$ with 
\begin{align}
\lambda\geq2\max\limits_{g\in\{1,...,G\}}\sigma\frac{\bigg[\Big(\epsilon_gp_g^{(1/\alpha_g-1/2)_+}+(1-\epsilon_g)p_g^{1/2}\Big)\Big(\big[p_{g}^{(1/2-1/\alpha_g)_+}\sqrt{np_g}\big] \bigwedge \frac{{p_g}^{1/\alpha_g^*}\sqrt{2n\log{p_g}}}{{p_g}^{1/\alpha_g^*}(1-\epsilon_g)+\epsilon_g} \Big)+  \sqrt{6n\log{G}}\bigg]}{\tau+(1-\tau)w_g}
\end{align}
satisfies the following $L_2$ error with probability at least $1-\frac{2}{G^2}-ce^{-c'/n}$ for some positive constants $c,c'$,
\begin{align}
\|\hat{\beta}-\beta^*\|_2^2\leq \frac{4\lambda^2\Big[\tau \sqrt{s}+(1-\tau) \sqrt{s_G}\max\limits_{g\in\{1,...,G\}}[w_gp_g^{(1/\alpha-1/2)_+}]\Big]^2}{\Big(\frac{\sqrt{n}}{4}\lambda_{\min}(\Sigma^{1/2})-3\kappa_1\kappa_2\Big)^4}.
\end{align}
where 
\begin{align*}
\kappa_1&=\max\limits_{g\in\{1,...,G\}} \Bigg[\frac{p_{g}^{(1/2-1/\alpha_g)_+}\sqrt{\Tr{(\Sigma_{(g)(g)})}}}{\tau+(1-\tau)w_g}\bigwedge \frac{3\sqrt{\log{p_g}\cdot\max_{i\in (g)}\Sigma_{ii}}}{(\tau+(1-\tau)w_g)(1-\epsilon_g+\epsilon_g p_g^{1/\alpha_g-1})}\Bigg],\\
\kappa_2&= \sqrt{G}\max\limits_{g\in\{1,...,G\}}(\tau p_g^{1/2}+(1-\tau)w_g p_g^{(1/\alpha_g-1/2)_+}),\\ 
\epsilon_g&=\frac{(1-\tau)w_g}{\tau+(1-\tau)w_g},\forall g\in\{1,...,G\},\ n> \frac{144\kappa_1^2\kappa_2^2}{\lambda_{\min}(\Sigma)}.
\end{align*}
\end{theorem*}
\begin{proof}[Proof of Theorem \ref{thm_l2error}]
Denote $h=\hat{\beta}-\beta^*$ and $S$ and $T$ are support set and group support set of $\beta^*$,i.e. 
\begin{align}
S &= \{i:\beta^*_i\neq  0\},\\
T &= \{g:\beta^*_{(g)}\neq  0\}.
\end{align}
Then we shall have 
\begin{alignat}{2}
&\|\beta^*\|_1&&=\|\beta^*_S\|_1\\
&\|\beta^*\|_{1,\alpha}&&=\|\beta^*_{(T)}\|_{1,\alpha}\\
&\|\beta^*\|_\ds&&=\tau\|\beta^*_S\|_1+(1-\tau)\sum\limits_{g\in T}^Gw_g\|\beta^*_{(g)}\|_{\alpha_g}
\end{alignat}
We further choose $\lambda$ such that $\lambda \geq 2\|X^\top\varepsilon\|_\ds^*$ and denote 
\begin{align}
\kappa_l = \Big(\frac{\sqrt{n}}{4}\lambda_{\min}(\Sigma^{1/2})-3\kappa_1\kappa_2\Big).
\end{align} Then according to Theorem \ref{thm_rec}, we have with probability at least $1-c\exp(-c'/n)$,
\begin{align}
\|Xv\|_2\geq \kappa_l\|v\|_2, v\in\mathbb{R}^p
\end{align}
It is easy to verify that 
\begin{align}
\|\beta^*+h\|_\ds &=\tau \|\beta^*+h\|_1+(1-\tau)\sum\limits_{g=1}^Gw_g\|\beta^*_{(g)}+h_{(g)}\|_{\alpha_g}\\
&=\tau \|\beta^*_S+\beta^*_{S^c}+h_S+h_{S^c}\|_1+(1-\tau)\sum\limits_{g\in T}w_g\|\beta^*_{(g)}+h_{(g)}\|_{\alpha_g}+(1-\tau)\sum\limits_{g\notin T}w_g\|\beta^*_{(g)}+h_{(g)}\|_{\alpha_g}\\
&=\tau \|\beta^*_S+h_S\|_1+\tau\|\beta^*_{S^c}+h_{S^c}\|_1+(1-\tau)\sum\limits_{g\in T}w_g\|\beta^*_{(g)}+h_{(g)}\|_{\alpha_g}\notag\\
&\quad\quad\quad\quad\quad\quad\quad\quad\quad\quad\quad\quad\quad\quad\quad+(1-\tau)\sum\limits_{g\notin T}w_g\|h_{(g)}\|_{\alpha_g}\\
&=\tau \|\beta^*_S+h_S\|_1+\tau\|h_{S^c}\|_1+(1-\tau)\sum\limits_{g\in T}w_g\|\beta^*_{(g)}+h_{(g)}\|_{\alpha_g}+(1-\tau)\sum\limits_{g\notin T}w_g\|h_{(g)}\|_{\alpha_g}
\end{align}
Therefore, 
\begin{align}
\|\beta^*+h\|_\ds-\|\beta^*\|_\ds=&\tau(\|\beta^*_S+h_S\|_1-\|\beta^*_S\|_1)+\tau\|h_{S^c}\|_1\notag\\
&+(1-\tau)\sum\limits_{g\notin T}w_g\|h_{(g)}\|_{\alpha_g}+(1-\tau)\sum\limits_{g\in T}w_g\big(\|\beta^*_{(g)}+h_{(g)}\|_{\alpha_g}-\|\beta^*_{(g)}\|_{\alpha_g}\big)\\
&\leq \tau(\|h_S\|_1+\|h_{S^c}\|_1)+(1-\tau)\sum\limits_{g\in T } w_g\|h_{(g)}\|_{\alpha_g}+(1-\tau)\sum\limits_{g\notin T } w_g\|h_{(g)}\|_{\alpha_g}\\
&=\tau(\|h_S\|_1+\|h_{S^c}\|_1)+(1-\tau)\sum\limits_{g=1}^G w_g\|h_{(g)}\|_{\alpha_g}\\
\|\beta^*+h\|_\ds-\|\beta^*\|_\ds=&\tau(\|\beta^*_S+h_S\|_1-\|\beta^*_S\|_1)+\tau\|h_{S^c}\|_1\notag\\
&+(1-\tau)\sum\limits_{g\notin T}w_g\|h_{(g)}\|_{\alpha_g}+(1-\tau)\sum\limits_{g\in T}w_g\big(\|\beta^*_{(g)}+h_{(g)}\|_{\alpha_g}-\|\beta^*_{(g)}\|_{\alpha_g}\big)\\
&\geq \tau(-\|h_S\|_1+\|h_{S^c}\|_1)+(1-\tau)[-\sum\limits_{g\in T } w_g\|h_{(g)}\|_{\alpha_g}+\sum\limits_{g\notin T } w_g\|h_{(g)}\|_{\alpha_g}]
\end{align}

Since $\hat{\beta}\in \argmin_\beta \{\|y-X\beta\|_2^2+\|\beta\|_\ds\}$,
\begin{align}
0&\geq \|y-X\hat{\beta}\|_2^2-\|y-X\beta^*\|_2^2+\lambda\|\beta^*+h\|_\ds-\lambda\|\beta^*\|_\ds \\
&= -2\langle X^\top\varepsilon,h\rangle+\|Xh\|_2^2+\lambda(\|\beta^*+h\|_\ds-\|\beta^*\|_\ds)\\
&\geq -2\langle X^\top\varepsilon,h\rangle+\kappa_l^2\|h\|_2^2+\lambda\Big(\tau(-\|h_S\|_1+\|h_{S^c}\|_1)+(1-\tau)[-\sum\limits_{g\in T } w_g\|h_{(g)}\|_{\alpha_g}+\sum\limits_{g\notin T } w_g\|h_{(g)}\|_{\alpha_g}]\Big)\\
&\geq -2\|X^\top\varepsilon\|_\ds^*\|h\|_\ds+\kappa_l^2\|h\|_2^2+\lambda\Big(\tau(-\|h_S\|_1+\|h_{S^c}\|_1)+(1-\tau)[-\sum\limits_{g\in T } w_g\|h_{(g)}\|_{\alpha_g}+\sum\limits_{g\notin T } w_g\|h_{(g)}\|_{\alpha_g}]\Big)\\
&\geq -\lambda\|h\|_\ds+\kappa_l^2\|h\|_2^2+\lambda\Big(\tau(-\|h_S\|_1+\|h_{S^c}\|_1)+(1-\tau)[-\sum\limits_{g\in T } w_g\|h_{(g)}\|_{\alpha_g}+\sum\limits_{g\notin T } w_g\|h_{(g)}\|_{\alpha_g}]\Big)\\
&\geq -\lambda(\tau\|h_{S}\|_1+\tau\|h_{S^c}\|_1+(1-\tau)\big(\sum\limits_{g\in T } w_g\|h_{(g)}\|_{\alpha_g}+\sum\limits_{g\notin T } w_g\|h_{(g)}\|_{\alpha_g}])\big)+\kappa_l^2\|h\|_2^2\notag\\
&\quad +\lambda\Big(\tau(-\|h_S\|_1+\|h_{S^c}\|_1)+(1-\tau)[-\sum\limits_{g\in T } w_g\|h_{(g)}\|_{\alpha_g}+\sum\limits_{g\notin T } w_g\|h_{(g)}\|_{\alpha_g}])\Big)\\
&\geq -2\lambda(\tau\|h_{S}\|_1+(1-\tau)\sum\limits_{g\in T } w_g\|h_{(g)}\|_{\alpha_g})+\kappa_l^2\|h\|_2^2\\
&\geq -2\lambda\Big[\tau \sqrt{s}+(1-\tau) \sqrt{s_G}\max\limits_{g\in\{1,...,G\}}[w_gp_g^{(1/\alpha-1/2)_+}]\Big]\|h\|_2+\kappa_l^2\|h\|_2^2
\end{align}
The last two lines are given using the fact 
\begin{align}
\|h_S\|_1 &=\sqrt{s}\|h_S\|_2,\|h_S\|_2\leq \|h\|_2\\
\sum\limits_{g\in T} w_g\|h_{(g)}\|_{\alpha_g}&\leq \sum\limits_{g\in T } w_gp_g^{(1/\alpha-1/2)_+}\|h_{(g)}\|_{2}\leq \max\limits_{g\in\{1,...,G\}}(w_gp_g^{(1/\alpha-1/2)_+})\sum\limits_{g\in T } \|h_{(g)}\|_{2}\\
&\leq \max\limits_{g\in\{1,...,G\}}(w_gp_g^{(1/\alpha-1/2)_+})\sqrt{s_G} \|h\|_{2}
\end{align}
Then we have 
\begin{align}
\|h\|_2\leq \frac{2\lambda\big[\tau \sqrt{s}+(1-\tau) \sqrt{s_G}\max\limits_{g\in\{1,...,G\}}[w_gp_g^{(1/\alpha-1/2)_+}]\big]}{\Big(\frac{\sqrt{n}}{4}\lambda_{\min}(\Sigma^{1/2})-3\kappa_1\kappa_2\Big)^2}.
\end{align}
with 
\begin{align*}
\kappa_1&=\max\limits_{g\in\{1,...,G\}} \Bigg[\frac{p_{g}^{(1/2-1/\alpha_g)_+}\sqrt{\Tr{(\Sigma_{(g)(g)})}}}{\tau+(1-\tau)w_g}\bigwedge \frac{3\sqrt{\log{p_g}\cdot\max_{i\in (g)}\Sigma_{ii}}}{(\tau+(1-\tau)w_g)(1-\epsilon_g+\epsilon_g p_g^{1/\alpha_g-1})}\Bigg]\\
\kappa_2&= \sqrt{G}\max\limits_{g\in\{1,...,G\}}(\tau p_g^{1/2}+(1-\tau)w_g p_g^{(1/\alpha_g-1/2)_+}).
\end{align*}

According to Theorem \ref{thm_lambda}, $\|X^\top \varepsilon\|_\ds^*$ satisfies the following tail probability bound, 
\begin{align}
&P\Bigg(\frac{\|X^\top \varepsilon\|_\ds^*}{n}\geq \max\limits_{g\in\{1,...,G\}}\sigma\frac{\bigg[\frac{\epsilon_gp_g^{(1/\alpha_g-1/2)_+}+(1-\epsilon_g)p_g^{1/2}}{\sqrt{n}}\Big(\big[p_{g}^{(1/2-1/\alpha_g)_+}\sqrt{p_g}\big] \bigwedge \frac{{p_g}^{1/\alpha_g^*}\sqrt{2\log{p_g}}}{{p_g}^{1/\alpha_g^*}(1-\epsilon_g)+\epsilon_g} \big)+  \sqrt{\frac{6\log{G}}{n}}\bigg]}{\tau+(1-\tau)w_g}\Bigg) \notag\\
\leq &\frac{2}{G^2}.
\end{align}

Then if we choose $\lambda$ such that 
\begin{align}
\lambda\geq2\max\limits_{g\in\{1,...,G\}}\sigma\frac{\bigg[\Big(\epsilon_gp_g^{(1/\alpha_g-1/2)_+}+(1-\epsilon_g)p_g^{1/2}\Big)\Big(\big[p_{g}^{(1/2-1/\alpha_g)_+}\sqrt{np_g}\big] \bigwedge \frac{{p_g}^{1/\alpha_g^*}\sqrt{2n\log{p_g}}}{{p_g}^{1/\alpha_g^*}(1-\epsilon_g)+\epsilon_g} \Big)+  \sqrt{6n\log{G}}\bigg]}{\tau+(1-\tau)w_g}
\end{align}
then the expression (\ref{eqn_l2error}) for $\|h\|_2$ will hold with probability at least $1-\frac{2}{G^2}-c\exp(-c'/n)$.
\end{proof}
\begin{lemma}
For $1\leq p<q\leq \infty$, $\|x\|_q\leq \|x\|_p\leq n^{1/p-1/q}\|x\|_q$ for $x\in \mathbb{R}^n$. 
\end{lemma}
\begin{lemma}[Lemma 3 of \cite{raskuttiRestrictedEigenvalueProperties}]\label{peeling}
Suppose that $A$ is some nonempty set in $\mathbb{R}^p$ and 

(1). $g:\mathbb{R}\to \mathbb{R}$ is non-negative and strictly increasing where $g(r)\geq \mu,\forall r\geq 0 $.

(2). $h:\mathbb{R}^p\to \mathbb{R}$ is a non-negative and increasing function. 

(3). There exists some constant $c>0$ such that for all $r>0$, the following tail bounds hold for some $a_n>0$
\begin{align}
P(\sup\limits_{v\in A,h(nu)\leq r}f(v,X)\geq g(r))\leq 2\exp(-ca_ng^2(r)).
\end{align}
We have 
\begin{align}
P(\mathcal{E})\leq \frac{2\exp(-4ca_n\mu^2)}{1-\exp(-4ca_n\mu^2)}, \mathcal{E}:=\{\exists v \in A, \ \textup{s.t.}\ f(v,X)\geq 2g(h(v))\}.
\end{align}
\end{lemma}
\begin{lemma}[Theorem 3.8 from \cite{massartConcentrationInequalitiesModel2007}]\label{lip}
Let $\varepsilon\sim\mathcal{N}(0,I_{n})$. Then for any $L$-Lipschitz function $f$, we have 
\begin{align}
P(\big|f(\varepsilon)-\mathbb{E}[f(\varepsilon)]\big|\geq \delta)\leq 2\exp(-\frac{\delta^2}{2L^2}),\ \forall \delta\geq0.
\end{align}
\end{lemma}
\begin{lemma}\label{lemma_l12norm}
The following lemma will be used throughout the rest of the paper. 
\begin{enumerate}
  \item Given group index $G_1$ and $G_2$ and $G_1\bigcap G_2=\emptyset$, 
  \begin{align}
  \|x_{(G_1)}+x_{(G_2)}\|_{1,2} =\|x_{(G_1)}\|_{1,2}+\|x_{(G_2)}\|_{1,2},\forall x.
  \end{align}
  \item Given any group index $G_1$ and $G_2$, 
  \begin{align}
  \big|\|x_{(G_1)}\|_{1,2}-\|x_{(G_2)}\|_{1,2}\big| &\leq \|x_{(G_1)}+x_{(G_2)}\|_{1,2}  \\
  &\leq\|x_{(G_1)}\|_{1,2}+\|x_{(G_2)}\|_{1,2},\forall x.
  \end{align}
  \item For any positive integer $d$ and vectors $x_1,x_2,...,x_d$, 
  \begin{align}
  \|\sum\limits_{i=1}^d x_i\|_2\leq \sum\limits_{i=1}^d\|x_i\|_{1,2}\leq \sqrt{d}\|\sum\limits_{i=1}^d x_i\|_2. 
  \end{align}
  In particular, if those $d$ vectors belong to $d$ different groups, we will have 
  \begin{align}
  \|\sum\limits_{i=1}^d x_i\|_2\leq \sum\limits_{i=1}^d\|x_i\|_{2}=\|\sum\limits_{i=1}^dx_i\|_{1,2}\leq \sqrt{d}\|\sum\limits_{i=1}^d x_i\|_2. 
  \end{align}
  \item For any vector $x$ with $d$ groups, 
  \begin{align}
  \|x\|_2\leq \frac{\|x\|_{1,2}}{\sqrt{d}}+\sqrt{d}\frac{\|x\|_{\infty,2}}{4}.
  \end{align}
\end{enumerate}
\end{lemma}

\begin{proof}[Proof of Lemma \ref{lemma_l12norm}]
Based on the definition of $\|\cdot\|_{1,2}$, it is easy to verify that Item 1 trivially holds. Item 2 can be derived based on triangle inequality. For Item 3, triangle inequality gives the first inequality. Cauchy Schwartz inequality can be used to prove $\|\sum\limits_{i=1}^dx_i\|_{1,2}\leq \sqrt{d}\|\sum\limits_{i=1}^d x_i\|_2$. 

Given $x = \big(x_{(1)},...,x_{(d)}\big)$ and further denote $x'_1=\|x_{(1)}\|_2,...,x'_{G}=\|x_{(G)}\|_2$. Consider a $d$-dimensional vectors $\bm{x}'=\big(x'_1,...,x'_{d}\big)$. It is easy to see that $\|x\|_2=\|\bm{x}'\|_2,\|x\|_{1,2}=\|\bm{x}'\|_1$. According to Proposition 1 of \cite{caiNewBoundsRestricted2010}, $\|\bm{x}'\|_2\leq \frac{\|\bm{x}'\|_1}{\sqrt{d}}+\sqrt{d}\frac{\max\limits_{1\leq i\leq d} x'_i-\min\limits_{1\leq i\leq d} x'_i}{4}$. We further have $\frac{\|\bm{x}'\|_1}{\sqrt{d}}+\sqrt{d}\frac{\max\limits_{1\leq i\leq d} x'_i-\min\limits_{1\leq i\leq d} x'_i}{4}=\frac{\|x\|_{1,2}}{\sqrt{d}}+\sqrt{d}\frac{\max\limits_{1\leq i\leq d} \|x_{(i)}\|_2-\min\limits_{1\leq i\leq d} \|x_{(i)}\|_2}{4}$. It concludes the proof.
\end{proof}

\end{document}